\documentclass[12pt,a4paper]{article}
\usepackage{amsmath,amssymb,amsthm}
\usepackage{mathtools}
\usepackage{booktabs}
\usepackage{graphicx}
\usepackage[margin=2.5cm]{geometry}
\usepackage{natbib}
\usepackage{setspace}
\usepackage{xcolor}

\usepackage{times}
\usepackage{comment}

\newtheorem{theorem}{Theorem}
\newtheorem{proposition}{Proposition}
\newtheorem{corollary}{Corollary}
\newtheorem{lemma}{Lemma}
\newtheorem{definition}{Definition}
\newtheorem{assumption}{Assumption}

\DeclareMathOperator{\KL}{KL}
\DeclareMathOperator{\Ber}{Ber}
\DeclareMathOperator{\Bin}{Bin}
\DeclareMathOperator*{\argmin}{argmin}
\DeclareMathOperator*{\argmax}{argmax}

\newcommand{\E}{\mathbb{E}}
\newcommand{\R}{\mathbb{R}}
\renewcommand{\Pr}{\mathbb{P}}
\newcommand{\cB}{\mathcal{B}}
\newcommand{\cT}{\mathcal{T}}
\newcommand{\MLE}{\mathrm{MLE}}
\newcommand{\RABIC}{\mathrm{RA\text{-}BIC}}
\newcommand{\BIC}{\mathrm{BIC}}

\title{
{\bf Information Gap and Feasibility-Aware Inference \\ 
in Binomial Logistic Mixtures}
}

\author{}
\date{}

\begin{document}
\maketitle
\doublespacing

\vspace{-1.5cm}

\begin{center}
{\large 
Yuta Hayashida$^1$ and Shonosuke Sugasawa$^2$
}

\vspace{5mm}
$^1$Graduate School of Economics, Keio University\\
$^2$Faculty of Economics, Keio University 
\end{center}

\vspace{2mm}
\begin{abstract}
This paper studies the information gap between mixture detection and label recovery in binomial logistic mixtures. Standard likelihood-based criteria such as the Bayesian information criterion (BIC) can detect the presence of two components, but this does not guarantee that the corresponding labels are recoverable. We show that this gap is intrinsic to binomial logistic mixtures with a fixed number of trials: observed-data evidence for mixture structure and per-observation information for label recovery have different local orders in the component separation, and only the former accumulates with the sample size. As a result, there exists a detectable-but-unrecoverable regime in which BIC selects two components while the posterior labels remain essentially uninformative. To address this issue, we propose two feasibility-aware inference procedures: a recoverability-aware BIC with a posterior-entropy penalty and an entropy-regularized estimator that mitigates the tendency of the maximum likelihood estimator to produce overly separated components and overly concentrated posterior responsibilities. Numerical experiments confirm the predicted gap and demonstrate that the proposed methods avoid misleading component selections and improve the calibration of posterior label probabilities.

\end{abstract}



\section{Introduction}
\label{sec:intro}

Finite mixture models provide a flexible framework for modeling unobserved heterogeneity in regression settings \citep{mclachlan2000finite,grun2004flexmix}. 
Mixtures of generalized linear models, including logistic and multinomial-logit mixtures, have been widely used for this purpose, and their identifiability has been studied in several settings \citep{follmann1991logistic,grun2008identifiability,auder2018binary}. 
In practice, however, logistic mixtures are often difficult to interpret: fitted components can overlap substantially, model selection can be unstable, and the resulting labels may not be informative. 
The difficulty is that evidence for more than one component in the observed-data distribution does not necessarily imply that individual latent labels can be recovered. 
This issue is well known in model-based clustering, where a model that fits the marginal distribution well does not always yield meaningful clusters. 
Entropy-based criteria, such as the integrated completed likelihood \citep{biernacki2000assessing} and related methods \citep{baudry2010combining}, penalize posterior uncertainty in the labels, whereas likelihood-based criteria such as the Bayesian information criterion \citep[BIC;][]{schwarz1978bic} mainly evaluate observed-data fit. 
For binomial logistic mixtures, however, it remains unclear why marginal fit, mixture-order selection, and label recovery can lead to different conclusions, and how this distinction should be reflected in inference.

Similar distinctions appear in related problems. 
In stochastic block models, community detection and exact recovery are governed by different information-theoretic thresholds \citep{abbe2018community}. 
In Gaussian mixture models, estimation and clustering have also been studied separately \citep{heinrich2018strong,wu2020optimal,ndaoud2022sharp}. 
In these settings, larger component separation provides more information for label recovery, leading to nontrivial recovery thresholds.
The fixed-trial binomial logistic mixture has a different feature. 
Because each response takes only a finite number of values, one observation contains only bounded information. 
Thus the evidence for a two-component marginal distribution cannot grow without bound as the components separate, and the information about an observation's own label is also bounded. 
Near the one-component boundary, label information is small, but a larger sample can still detect increasingly small departures from a one-component model. 
This creates a ``detectable-but-unrecoverable" regime, where the mixture structure is detectable in aggregate, while individual labels remain essentially unrecoverable. 
The gap is structural, not numerical, algorithmic, or a small-sample artifact. 
More observations improve detectability, but recoverability improves only through more per-observation information, such as larger separation or more trials.

To formalize this gap, we define a complete-data divergence that measures the information contributed by the latent labels. 
A two-component logistic mixture is useful for clustering only when the labels provide information beyond a one-component description. 
Our measure captures this idea by comparing the true two-component joint distribution of $(Y,Z)$ with the best one-component approximation that preserves the mixing weights. 
It is zero when the two components coincide and positive when the labels are informative.
We show that this divergence decomposes exactly into two terms. 
The first is a detectability term, which measures the improvement in the observed-data distribution. 
The second is a recoverability term, which measures the information that each observation carries about its own label. 
This decomposition is general and only requires a one-component approximation with the same mixing weights. 
Near the one-component boundary, the two terms have different local orders: detectability is fourth order in the component separation, whereas recoverability is second order. 
This order gap also appears in Gaussian mixtures. 
What is specific to the fixed-trial binomial setting is that detectability is bounded as the components separate and scales differently from recoverability in the number of trials.

The decomposition also explains why BIC can be inadequate as a clustering-oriented criterion. 
The likelihood gain in BIC measures improved observed-data fit, but not the recoverability of the corresponding latent labels. 
Thus, in the detectable-but-unrecoverable regime, BIC may select two components even when the posterior labels remain nearly uninformative beyond the prior class probabilities.
The same gap also affects estimation. 
While BIC can over-select at the model-selection stage, the maximum likelihood estimator can produce posterior responsibilities that are too concentrated at the estimation stage. 
We address these two issues with entropy-based feasibility-aware procedures: recoverability-aware BIC (RA-BIC) for model selection and entropy-regularized estimation (ER) for parameter estimation.

This paper is organized as follows. Section~\ref{sec:model} introduces the model and the information decomposition, contrasts the information geometry of logistic and Gaussian mixtures, and establishes the local laws and phase separation. 
Section~\ref{sec:inference} develops feasibility-aware model selection via RA-BIC and entropy-regularized estimation. Section~\ref{sec:experiments} presents numerical experiments. 
Section~\ref{sec:discussion} discusses connections and extensions. 
Proofs are collected in the Supplementary Material.


\section{Merge Divergence and Decomposition}
\label{sec:model}

\subsection{Model and notation}

We write $\psi(t) = 1/(1+e^{-t})$ for the logistic link. For the two-component binomial logistic mixture, we use the weighted-center parametrization. Let $Z \in \{1,2\}$ denote the latent class indicator with $\Pr(Z=1) = \pi$ and $\Pr(Z=2) = 1-\pi$. Conditional on $Z = k$ and the covariate $X = x \in \R^p$, the model is 
\begin{equation}\label{eq:mix}
Y | Z = k, X = x \;\sim\; \Bin(m, p_k(x)), \qquad p_k(x) = \psi(x^\top \beta_k), 
\end{equation}
where $m \geq 1$ is the known number of trials. 
We reparametrize the two coefficients through a shared center $\beta_0$ and a contrast $\Delta$, setting $\beta_1 = \beta_0 + (1-\pi)\Delta$ and $\beta_2 = \beta_0 - \pi\Delta$, so the overall parameter is $\theta = (\pi, \beta_0, \Delta) \in (0,1) \times \R^p \times \R^p$. These weights give $\pi\beta_1 + (1-\pi)\beta_2 = \beta_0$, balancing the two components around $\beta_0$ in the linear predictor. The balance cancels the first-order deviation in the local expansion of the merge divergence, which makes the leading detectability term fourth order in the separation.
The marginal success probability of $Y$ given $X = x$ is $\bar p_\theta(x) = \pi\psi(x^\top \beta_1) + (1-\pi)\psi(x^\top \beta_2)$. 
The one-component binomial logistic family is parametrized by $q_\beta(x) = \psi(x^\top \beta)$.
For population-level arguments, we suppress the observation index $i$ and write $(X,Y,Z)$ for a generic observation from the model, with $m$ denoting the corresponding number of trials.
We write $f_B(\cdot; m, p)$ for the probability mass function of the binomial distribution ${\rm Bin}(m, p)$.

\subsection{Complete-data divergence and its decomposition}
Fitting of the observed data alone can indicate whether the marginal of $Y$ is better represented by a mixture, but not whether labels are statistically informative. To assess whether labels add information beyond marginal distributional fit, we introduce the \emph{complete-data divergence} $\cT_m^*$, which incorporates both marginal detectability and per-observation label informativeness.

The joint distribution $P_\theta(Y, Z| X)$ of $Y$ and $Z$ given $X$ is 
\begin{equation}\label{eq:joint}
P_\theta(Y,Z| X) = \pi^{I(Z=1)}(1-\pi)^{I(Z=2)} 
f_B(Y; m, p_{Z}(X))
\end{equation}
We also define $Q_\gamma(Y, Z |  X)$ as the surrogate that preserves the mixing weights but replaces the success probability $p_{Z}(X)$ with  the common probability $q_\gamma(x) = \psi(x^\top \gamma)$ for some parameter $\gamma$, given by  
\[
Q_\gamma(Y, Z |  X) = \pi^{I(Z=1)}(1-\pi)^{I(Z=2)} 
f_B(Y; m, q_\gamma(X))
\]
 
Then, we define the complete-data Kullback--Leibler divergence between the true two-component model and its best one-component surrogate, as follows: 
\begin{equation}\label{eq:merge}
\cT_m^*(\theta) \equiv  \inf_{\gamma \in \R^p} \E_X\left[ \KL\bigl(P_\theta(Y,Z | X) \| Q_\gamma(Y,Z | X)\bigr) \right],
\end{equation}
where ${\rm KL}(\cdot \| \cdot)$ is the Kullback-Leibler (KL) divergence and the outer expectation is taken with respect to the marginal distribution of $X$. 
For brevity, we refer to $\cT_m^*(\theta)$ simply as the ``complete-data divergence".
The complete-data divergence measures the distance between the two-component mixture $P_\theta(Y,Z| X)$ and one-component surrogate $Q_\gamma(Y, Z |  X)$. 
$\cT_m^*(\theta) = 0$ if and only if the two components are identical ($\Delta = 0$), in which case the labels are redundant and a single-component model is sufficient with $\gamma=\beta_0$.
When $\cT_m^* > 0$, the labels carry genuine information, so $\cT_m^*$ measures at the population level whether the two-component model provides a useful clustering representation rather than only a better marginal fit.

Since the label $Z$ is not observed, it is useful to decompose $\cT_m^*$ into a component accessible from $Y$ alone and a component requiring the labels.
To this end, we define the \emph{detectability functional} as 
\begin{equation}\label{eq:D}
D_m^*(\theta) \equiv  \inf_{\gamma \in \R^p} \E\left[ \KL\left( P_\theta(Y | X) \|\Bin\bigl(m, q_\gamma(X)\bigr) \right) \right],
\end{equation}
where $P_\theta(Y | X)=P_\theta(Y,Z=1| X)+P_\theta(Y,Z=2| X)$ is the marginal probability. 
$D_m^*$ measures how distinguishable the mixture distribution of~$Y$ is from the best single-component fit. 
We also define the \emph{recoverability functional} as
\begin{equation}\label{eq:R}
R_m^*(\theta) \equiv \mathbb{E}\left[ \log\frac{P_\theta(Y,Z| X)}{P_\theta(Z) P_\theta(Y| X)}
\right].
\end{equation}
Since $P_\theta(Z)=P_\theta(Z|X)$ in our setting, the above quantity is the conditional mutual information $I_\theta(Z; Y | X)$, namely how much the observation $(X, Y)$ reduces uncertainty about its own label $Z$. 
An essential property of the merge divergence (\ref{eq:merge}) is the following decomposition: 

\begin{proposition}
\label{prop:decomposition}
The complete-data divergence $\cT_m^*(\theta)$ is decomposed as $\cT_m^*(\theta) = D_m^*(\theta) + R_m^*(\theta)$.
\end{proposition}

Thus the total information value of the labels decomposes into the detectability (detectable from the distribution of~$Y$) and the recoverability (the label information in each data point). This decomposition is what makes the gap operational. The BIC log-likelihood gain converges to $2nD_m^*$ (Section~\ref{sec:RA-BIC}), so among the two summands BIC tests only the detectability $D_m^*$ and is structurally blind to the recoverability $R_m^*$. 
This asymmetry motivates the feasibility-aware procedures developed in Section~\ref{sec:inference}.
In the subsequent section, we investigate the behavior of the two quantities as a function of $\rho\equiv (\Delta^\top \Sigma \Delta)^{1/2}$ with $\Sigma=\mathbb{E}[XX^\top]$, which measures the distance between $\beta_1$ and $\beta_2$. 

From the joint distribution (\ref{eq:joint}), the posterior distribution of $Z$ can be expressed as 
$$
P_\theta(Z|X,Y)=\frac{\{\pi f_B(Y; m, p_1(X))\}^{I(Z=1)} \{(1-\pi) f_B(Y; m, p_2(X))\}^{I(Z=2)} }{\pi f_B(Y; m, p_1(X)) + (1-\pi) f_B(Y; m, p_2(X))}.
$$
The recoverability functional admits an alternative expression based on the above posterior distribution, as follows: 

\begin{proposition}
\label{prop:entropy-representation}
The recoverability functional $R_m^*(\theta)$ can be expressed as $R_m^*(\theta) = h(\pi) - E_m(\theta)$, where $h(\pi) = -\pi\log\pi - (1-\pi)\log(1-\pi)$ is the binary entropy function and $E_m(\theta) = -\sum_{k=1}^{2} \mathbb{E}\left[P_\theta(Z{=}k|X,Y)\log P_\theta(Z{=}k|X,Y)\right]$ is the entropy of the posterior distribution $P_\theta(Z|X,Y)$ in the two-component model.
\end{proposition}

When $R_m^* = 0$, the posterior $P_\theta(Z | X, Y)$ equals the prior, so the data provides no information about the label. This entropy representation is central to the feasibility-aware component-selection and estimation procedures of Sections~\ref{sec:RA-BIC} and~\ref{sec:estimation}.

\subsection{Difficulty in binomial logistic mixtures: comparison with Gaussian mixtures}
\label{sec:gaussian}

Binomial logistic mixtures are often more difficult to estimate than Gaussian mixtures, and this section identifies a key reason, which is not the local order gap. Both families share the same quartic--quadratic gap, detectability fourth order and recoverability second order in the separation, so the order gap alone is not specific to the binomial model. 
The binomial-specific feature is the boundedness of the observed-data information. 
With a fixed number of trials the response $Y \in \{0,1,\dots,m\}$ carries at most $\log(m+1)$ bits, so the observed-data evidence for a second component cannot grow without bound as the components separate, unlike in a Gaussian mixture. 
This difference changes the nature of the inference problem, which we make precise by comparing the information functionals $D_m^*$ and $R_m^*$ across the two families.
To this end, we consider the Gaussian mixture, $Y_i | Z_i = k, X_i \sim N(X_i^\top\beta_k, \sigma^2)$ with common variance $\sigma^2$.
For this model, the KL-optimal one-component surrogate within the homoscedastic Gaussian regression family follows by moment matching \cite{herbrich2005minimizingkl}. Since
$\mathbb{E}(Y | X)=X^\top\{\pi\beta_1+(1-\pi)\beta_2\}$ and $\mathrm{Var}(Y | X)=\sigma^2+\pi(1-\pi)(X^\top\Delta)^2,$ the corresponding one-component model is $N(X^\top\beta_*, \sigma_*^2)$, where 
\[
\beta_*=\pi\beta_1+(1-\pi)\beta_2,
\qquad
\sigma_*^2=\sigma^2+\pi(1-\pi)\E\left[(X^\top\Delta)^2\right]
=
\sigma^2+\pi(1-\pi)\rho^2,
\]
and $\rho \equiv (\Delta^\top \E[XX^\top]\Delta)^{1/2}$. Denote by $D_{\rm Gauss}^*$ and $R_{\rm Gauss}^*$ the Gaussian-mixture analogues of (\ref{eq:D}) and (\ref{eq:R}).

\begin{proposition}\label{prop:rate}
Under $\rho\to\infty$, $D_{\mathrm{Gauss}}^*\to \infty$ and $R_{\mathrm{Gauss}}^*$ is bounded for Gaussian mixtures, while both $D_m^*$ and $R_m^*$ are bounded for binomial logistic mixtures.
\end{proposition}

The saturation of binomial $D_m^*$ (unlike diverging Gaussian $D_{\rm Gauss}^*$) is the key difference.
Evidence for detection ($nD_m^*$) accumulates with the sample size $n$ but the sample size does not help per-observation recoverability $R_m^*$. 
The trial count $m$ raises both, but detection super-linearly (via binomial over-dispersion, scaling with trial pairs) and recovery only linearly. Thus larger $n$ detects mixtures at smaller separations where $R_m^*$ is smallest. Recovery requires larger $m$ or separation, not more observations. The $n$-versus-$m$ asymmetry is the operational core of the difficulty.

Figure~\ref{fig:Dcurve-msweep} plots $D_m^*$ and $R_m^*$ versus separation $\delta = |\beta_1 - \beta_2|$ ($\tilde{X} \sim N(0,1)$, $\pi = 0.5$, $m \in \{1, 2, 5, 10\}$, $\sigma^2 \in \{1.0, 0.1\}$). Gaussian $D^*$ diverges while logistic $D^*$ saturates. For small $m$ saturation is low, but $m=10$ exceeds Gaussian ($\sigma^2=1$) at moderate separations due to mixture binomial over-dispersion. Both families approach $h(\pi) = \log 2$ for recoverability, with logistic curves slower for small $m$.

\begin{figure}[t]
\centering
\includegraphics[width=0.48\textwidth]{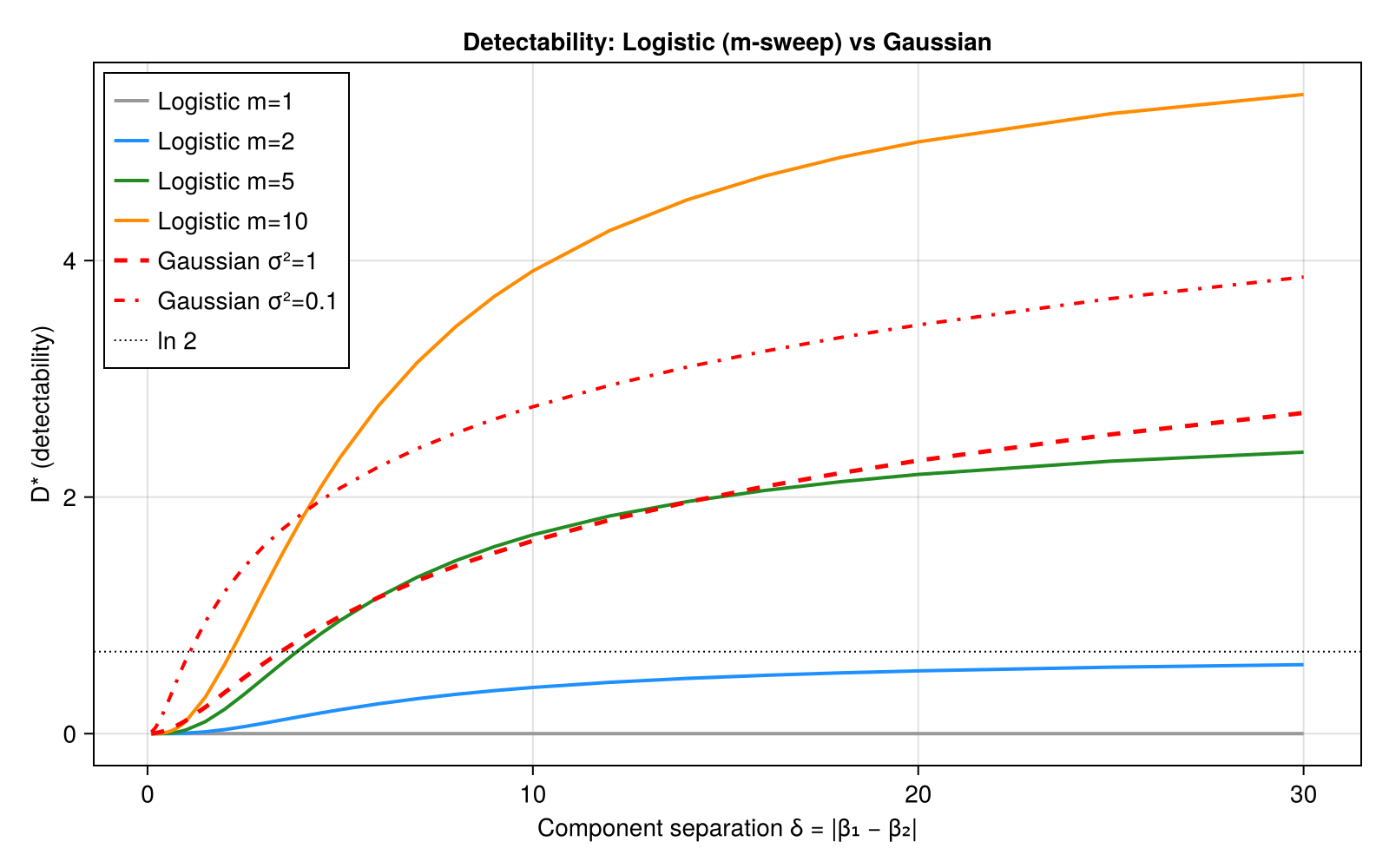}
\hfill
\includegraphics[width=0.48\textwidth]{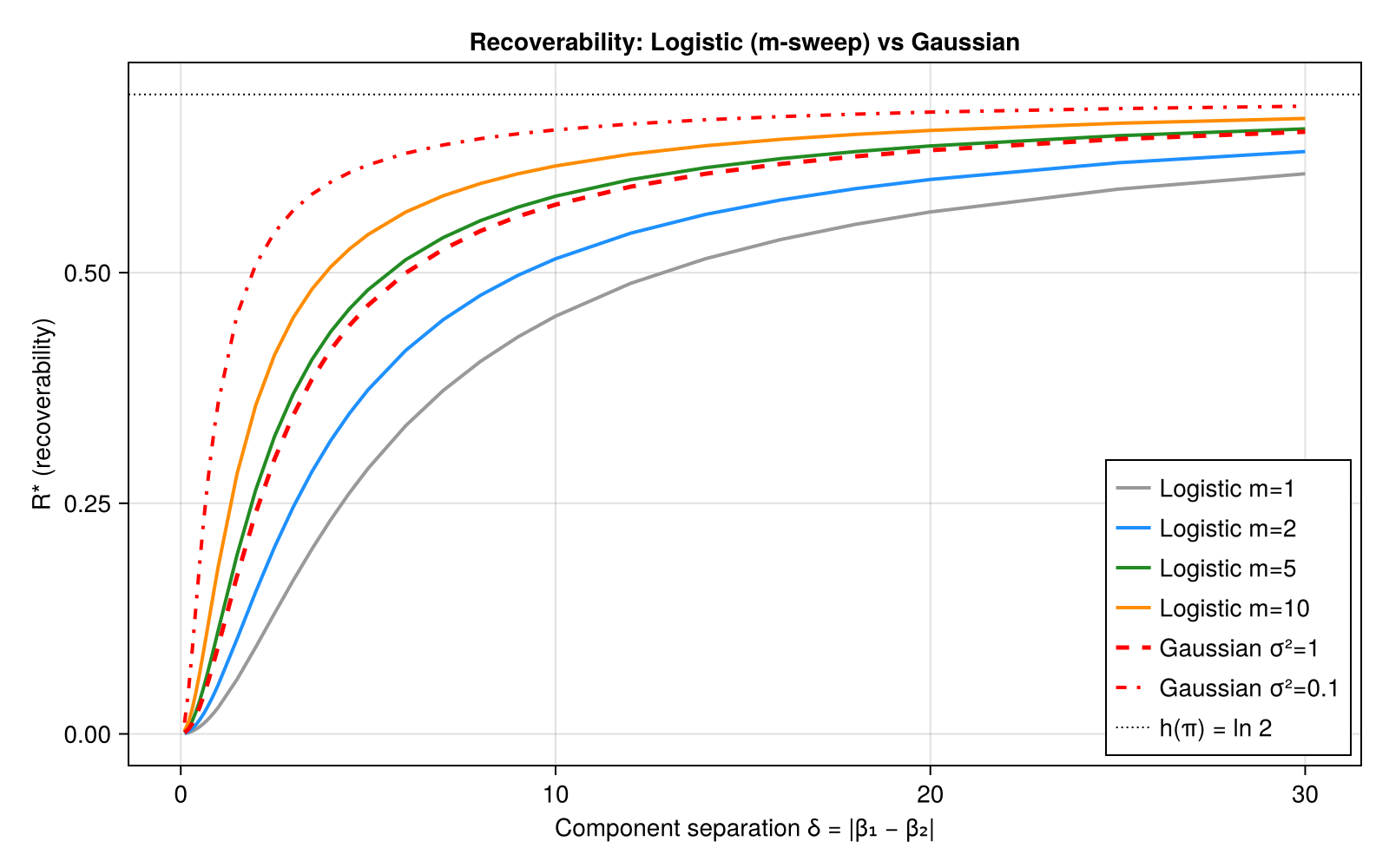}
\caption{Detectability $D_m^*$ (left) and recoverability $R_m^*$ (right) versus component separation $\delta = |\beta_1 - \beta_2|$, for logistic mixtures with $m \in \{1, 2, 5, 10\}$ (solid) and Gaussian mixtures with $\sigma^2 \in \{1, 0.1\}$ (dashed); illustrative one-dimensional setting, $\tilde{X} \sim N(0,1)$, $\pi = 0.5$. Gaussian $D^*$ diverges while logistic $D_m^*$ saturates (lower for smaller $m$); both $R_m^*$ approach $h(\pi) = \log 2$. }
\label{fig:Dcurve-msweep}
\end{figure}

\subsection{Information gap }
\label{sec:local}

This section establishes how the detectability $D_m^*$ and recoverability $R_m^*$ behave in the local merge regime (i.e. small $\rho$). 
The quartic law for $D_m^*$ and the quadratic law for $R_m^*$ together exhibit the detectable-but-unrecoverable regime and the structural asymmetry between sample size $n$ and trial count $m$.
We make the following assumptions: 

\begin{itemize}
\item[(A1)]
There exist $\varepsilon \in (0,1/2)$ and a compact set $\cB \subset \R^p$ such that $\pi \in [\varepsilon, 1-\varepsilon]$ and $\beta_0 \in \cB$.

\item[(A2)]
$\E[XX^\top]$ is positive definite and $\E[\|X\|^8] < \infty$.

\item[(A3)]
The response probabilities satisfy $0 < \psi(X^\top\beta) < 1$ for $\beta$ in a neighborhood of $\cB$, and standard moment and integrability conditions hold (detailed in Assumption~\ref{ass:SM-technical} in the Supplementary Material). 
\end{itemize}

The cancellation of the first-order term in the mixture mean under the weighted-center parametrization makes the leading KL term fourth order.

\begin{lemma}[Decomposition of the observed-data KL]
\label{lem:kl-decomposition}
For any $\gamma \in \R^p$, it holds that 
\begin{align*}
&\E_X\left[\KL\bigl(P_\theta(Y | X) \big\| \Bin(m, q_\gamma(X))\bigr)\right] \\
& \quad = \E_X\left[\KL\bigl(P_\theta(Y | X) \big\| \Bin(m, \bar p_\theta(X))\bigr)\right] + m \E_X\left[\KL\bigl(\Ber(\bar p_\theta(X)) \big\| \Ber(q_\gamma(X))\bigr)\right].
\end{align*}
\end{lemma}

The first term on the right-hand side is the KL divergence between the mixture binomial and the binomial with matched mean. It vanishes when $m=1$, since a Bernoulli is fully determined by its mean. This is why $J_1 = \kappa$ is much smaller than $J_m$ for $m \geq 2$ in Table~\ref{tab:scaling}, and why the gap between detectability and recoverability is hardest to demonstrate at $m=1$ in the experiments of Section~\ref{sec:experiments}.
\begin{theorem}[Expansion of detectability]
\label{thm:true-quartic}
Under Assumptions~(A1)--(A3), $D_m^*(\theta) = J_m\rho^4 + o(\rho^4)$ as $\rho\to 0$, where
\[
J_m = m\kappa + \frac{m(m-1)}{4}\bar\lambda,
\]
with
\begin{align*}
\kappa &= \frac{1}{2}\inf_{b \in \R^p} \E\left[\frac{\big\{2^{-1}\pi(1-\pi)\psi''(X^\top\beta_0)(\rho^{-1}\Delta^\top X)^2 \;-\; \psi'(X^\top\beta_0)X^\top b\big\}^2}{\psi'(X^\top\beta_0)}\right], \\
\bar\lambda &= \bigl[\pi(1-\pi)\bigr]^2 \E\left[\psi'(X^\top\beta_0)^2(\rho^{-1}\Delta^\top X)^4\right].
\end{align*}
\end{theorem}

The coefficient $m\kappa$ arises from the Bernoulli-level mean mismatch between the mixture mean $\bar p_\theta(X)$ and the best single-component fit, and $m(m-1)\bar\lambda/4$ arises from the distributional gap between the mixture binomial and any single binomial.
In particular, $J_1 = \kappa$ and $J_m \ge mJ_1$ for $m \ge 2$, with strict inequality whenever $\bar\lambda > 0$; the latter is automatic under (A2)--(A3) since $\psi'(X^\top\beta_0) > 0$ almost surely and $\E[XX^\top]$ is positive definite.

We next expand the recoverability functional in (\ref{eq:R}).

\begin{theorem}[Expansion of recoverability]
\label{thm:quadratic-recoverability}
Under Assumptions~(A1)-(A3), it holds that $R_m^*(\theta) = m\xi\rho^2 + o(\rho^2)$ as $\rho\to 0$, where
\[
\xi \equiv \frac{\pi(1-\pi)}{2} \E\left[\psi'(X^\top\beta_0)(\rho^{-1}\Delta^\top X)^2\right],
\]
is a nonnegative constant.
\end{theorem}

We write $A_m := m\xi$ for the per-observation recoverability coefficient, so that $R_m^*(\theta) = A_m \rho^2 + o(\rho^2)$ and $A_m = m A_1$. The exact linearity $A_m = m\xi$ is a statement about the leading $\rho^2$ coefficient. 
At fixed $\rho$ the recoverability $R_m^*$ itself is sub-linear in $m$, because the per-observation mutual information cannot exceed the prior label entropy $h(\pi) \le \log 2$, so the clean proportionality holds only in the local regime where $R_m^*$ is far below this ceiling. The numerical experiments in Section~\ref{sec:experiments} and the Supplementary Material assumption (R1) report $A_m$ and $A_m/m$ rather than $\xi$ to make the linear dependence on $m$ explicit.

From the alternative expression of $R_m^*(\theta)$ in Proposition~\ref{prop:entropy-representation}, we obtain $E_m(\theta) = h(\pi) - A_m\rho^2 + o(\rho^2)$ as $\rho \to 0$.
Thus, the posterior assignment uncertainty remains close to its maximum $h(\pi)$ unless $A_m$ is sufficiently large.
Even the oracle Bayes classifier then cannot beat the classifier based on the prior distribution.

\begin{proposition}[Vanishing recoverability and oracle classification]
\label{thm:oracle-impossibility}
If $R_m^*(\theta_n) \to 0$ and $\pi_n \to \pi_* \in (0,1)$, then the conditional posterior collapses to the prior:
\[
\E_{\theta_n}\left[ \| P_{\theta_n}(Z | X, Y) - P_{\theta_n}(Z) \|_{\mathrm{TV}} \right] \to 0.
\]
Consequently, the oracle Bayes classifier $\hat Z(X,Y) = \argmax_z P_{\theta_n}(Z = z | X, Y)$ attains only the trivial misclassification rate $\min(\pi_*, 1-\pi_*)$.
\end{proposition}

\subsection{Implications of the order gap}
\label{sec:implications}

As established in Section~\ref{sec:gaussian}, the quartic--quadratic order gap is shared with Gaussian mixtures and does not by itself single out the binomial model.
What is specific to the binomial model is the interaction of this order gap with bounded observed-data information.
The contrast between the quartic law $D_m^*(\theta) = J_m\rho^4 + o(\rho^4)$ (Theorem~\ref{thm:true-quartic}) and the quadratic law $R_m^*(\theta) = m\xi\rho^2 + o(\rho^2)$ (Theorem~\ref{thm:quadratic-recoverability}) has two direct consequences. First, locally $D_m^* / R_m^* \to 0$, so detectability vanishes faster than recoverability as $\rho \to 0$. Second, the BIC log-likelihood gap $2nD_m^*$ aggregates linearly in $n$, while $R_m^*$ has no $n$-dependence. The asymmetry is thus one of aggregation rather than of per-observation magnitude. We summarize the consequence in the following corollary.

\begin{corollary}[Detectable-but-unrecoverable sequences]
\label{cor:phase-separation}
Under Assumptions~(A1)--(A3), there exists a sequence $\theta_n = (\pi, \beta_0, \Delta_n)$ with $\rho(\Delta_n) \to 0$ such that
\[
\frac{n D_m^*(\theta_n)}{\log n} \to \infty,
\qquad
R_m^*(\theta_n) \to 0.
\]
Consequently, any BIC-consistent selector between $K=1$ and $K=2$ chooses $K=2$ with probability tending to one, while the posterior entropy $E_m(\theta_n)$ converges to the prior entropy $h(\pi)$ and individual labels remain asymptotically unrecoverable in the sense of Proposition~\ref{thm:oracle-impossibility}.
\end{corollary}

The drifting sequence $\rho_n \to 0$ in the corollary is obtained from the local laws alone and, in this sense, is not specific to the binomial model.
A similar local detectable-but-unrecoverable drift can also arise in Gaussian mixtures.
What is specific to the fixed-trial binomial model is the bounded amount of information in each observation.
For fixed $m$, both $D_m^*$ and $R_m^*$ are bounded (Proposition~\ref{prop:rate}).
The observed response has at most $\log(m+1)$ bits of information, so detectability cannot grow without bound as the components separate.
Recoverability is also bounded by $h(\pi)$, although it increases toward this bound as the separation grows.
Thus, collecting more observations does not by itself close the gap: increasing $n$ strengthens detectability through the accumulated likelihood, whereas $R_m^*$ is a per-observation quantity and does not depend on $n$.
The gap can be reduced only by increasing the information in each observation, through a larger separation or a larger trial count $m$.
This is the main difference from the Gaussian case, where detectability can continue to grow with separation.
Therefore, the binomial-specific feature is the combination of detectability saturation and the asymmetry between the sample size $n$ and the trial count $m$.


\section{Feasibility-Aware Inference}
\label{sec:inference}

The detectability--recoverability gap appears in two parts of finite-sample inference, at model selection (Section~\ref{sec:RA-BIC}) and at parameter estimation (Section~\ref{sec:estimation}). 
Both procedures are based on the entropy representation $R_m^*(\theta) = h(\pi) - E_m(\theta)$ and use the empirical posterior entropy $\widehat E_m$ at the common rate $\sqrt{\log n/n}$.
RA-BIC penalizes two-component selections with high posterior label uncertainty, whereas entropy-regularized estimation discourages overly concentrated posterior responsibilities and prevents the fitted model from overstating label recoverability.

\subsection{Feasibility-aware selection of the number of mixture components}\label{sec:RA-BIC}

We focus on the comparison between $K=1$ and $K=2$ throughout this section.
The question is whether to fit a single-component model or a two-component mixture. This is the setting in which the gap regime is easiest to characterize, and it admits a clean asymptotic treatment of the entropy penalty (we discuss the case with $K \geq 3$ at the end of this subsection). 
Our proposed criterion answers a different question from standard mixture order selection: it asks not which $K$ best fits the marginal distribution of $Y$, but which $K$ is supported as a \emph{clustering representation}, in the sense that the latent labels carry enough per-observation information to be recoverable. 
To name this target we introduce a quantity $K_n^{\mathrm{feas}}$ below.
This quantity should be interpreted as part of the decision rule, rather than as a separate population estimand.

\begin{definition}[Recoverability-feasible number of components]
\label{def:feasible-K}
Fix the trial count $m$ and consider a sequence of two-component parameters $\theta_n=(\pi_n,\beta_{0n},\Delta_n)$.
Let
$$
a_n(\theta_n)=\frac{1}{2}\sqrt{\frac{\log n}{n}}\big\{h(\pi_n)-R_m^*(\theta_n)\big\}.
$$
We define the recoverability-feasible number of components by $K_n^{\mathrm{feas}}(\theta_n)=2$ if $D_m^*(\theta_n)/a_n(\theta_n)\to\infty$ and $K_n^{\mathrm{feas}}(\theta_n)=1$ if $D_m^*(\theta_n)=o(\sqrt{\log n/n})$. 
The definition leaves unresolved the boundary case where $D_m^*(\theta_n)$ is of the same order as $a_n(\theta_n)$.
\end{definition}

For a fixed interior parameter $\theta_0$ with $D_m^*(\theta_0)$ and $R_m^*(\theta_0)$ bounded away from zero, the first condition holds and hence $K_n^{\mathrm{feas}}(\theta_0)=2$ for all sufficiently large $n$.
In contrast, for the local detectable-but-unrecoverable sequences considered in Theorem~\ref{thm:gap-rejection}, the second condition holds and $K_n^{\mathrm{feas}}(\theta_n)=1$.
Thus $K_n^{\mathrm{feas}}$ agrees with the usual generative order in the fixed recoverable regime, but differs from it along drifting sequences where a second component is detectable yet its labels are not recoverable.

Standard BIC selects the optimal $K\in \{1,2\}$ by minimizing
\[
\BIC(K) = -2\ell_n(\widehat\theta_{n,K}^{\MLE}) + \nu_K \log n,
\]
where $\ell_n(\theta) = \sum_{i=1}^n \log \bigl[\sum_{k=1}^K \pi_k f_B(Y_i; m,\psi(X_i^\top \beta_k))\bigr]$ is the log-likelihood, $\widehat\theta_{n,K}^{\MLE}$ is the maximum likelihood estimator (MLE) and $\nu_K\equiv pK+K-1$ is the number of unknown parameters in a $K$-component model.
When the two-component model is correctly specified, the maximized log-likelihood under $K=2$ converges to the population log-likelihood of the true model, while that under $K=1$ converges to the best one-component approximation.
Since $D_m^*$ is the per-observation Kullback--Leibler loss of the best one-component surrogate, we have $\ell_n(\hat\theta_{n,2})-\ell_n(\hat\theta_{n,1})=nD_m^*+r_n$, where $r_n=o_p(nD_m^*)$ is a remainder term. 
Then, it holds that 
$$
\mathrm{BIC}(1)-\mathrm{BIC}(2)=2nD_m^*-(\nu_2-\nu_1)\log n + 2r_n.
$$
Thus BIC selects \(K=2\) when the accumulated detectability \(nD_m^*\)
dominates the complexity penalty \((\nu_2-\nu_1)\log n/2\), up to the stochastic remainder.
This shows that BIC is driven by detectability $D_m^*$ and does not account for recoverability $R_m^*$.
Consequently, in the gap regime of Corollary~\ref{cor:phase-separation}, BIC over-selects: it chooses $K=2$ with probability tending to one, even though the resulting labels are no more informative than the prior class proportions at the individual-observation level.
The oracle impossibility result in Proposition~\ref{thm:oracle-impossibility} should be interpreted as a statement about average recoverability.
Since $R_m^*=\E_X[I(Z;Y| X)]$ averages the conditional label information over the covariate distribution, $R_m^*\to 0$ implies a collapse in average recoverability, not necessarily a uniform collapse over all covariate values.
Thus, a Chow-type reject rule that classifies only observations with large $I(Z;Y| X=x)$ and abstains otherwise may retain non-trivial accuracy on the retained subset even when $R_m^*\to 0$ \citep{chow1970,elyaniv2010selective}.

This qualification does not change the main conclusion for the local boundary regime.
As $\rho \to 0$, both the retained subset and its accuracy advantage vanish, so selective classification may be useful in the moderate gap regime but not at the boundary itself.
It is therefore preferable to choose $K$ using a criterion that accounts for recoverability as well as detectability.
Such a criterion should select $K=2$ only when the fitted mixture provides not only evidence of heterogeneity but also labels that are meaningful for individual-level recovery.
The integrated completed likelihood provides a natural reference point in the opposite direction: because it uses a fixed entropy penalty, it tends to be conservative.
The criterion proposed here is designed to lie between BIC and the integrated completed likelihood, suppressing the detectable-but-unrecoverable selections accepted by BIC while still selecting $K=2$ in recoverable settings where the integrated completed likelihood remains conservative.

To this end, we first consider an empirical proxy of $R_m^*(\theta)$. 
From Proposition~\ref{prop:entropy-representation}, $R_m^*(\theta) = h(\pi) - E_m(\theta)$, so the population posterior entropy $E_m(\theta)$ controls $R_m^*(\theta)$ and a sample counterpart of $E_m(\theta)$ is 
\begin{equation}\label{eq:entropy}
\widehat{E}_{m,K}(\theta)= -\sum_{i=1}^n \sum_{k=1}^K  P_{\theta}(Z_i=k|X_i,Y_i)\log P_{\theta}(Z_i=k|X_i,Y_i),
\end{equation}
noting that $\widehat{E}_{m,K}(\theta)=0$ when $K=1$.
Since $\widehat{E}_{m,K}(\theta) / n$ approximates $h(\pi) - R_m^*$, large $\widehat{E}_{m,K}(\theta)$ indicates low evidence of recoverability. The same identity yields a simple consistent feasibility diagnostic: plugging the estimated mixing proportion into the binary entropy and subtracting the normalized empirical entropy gives $\widehat R = h(\widehat\pi) - \widehat E_{m,2}(\widehat\theta)/n$, a sample estimate of $R_m^*$ that converges to it under (A5) and can flag a gap-regime fit alongside the selected model.
We do not claim that $\widehat R$ is debiased.
It inherits the gap-regime over-concentration of MLE, so it is a screening quantity rather than a calibrated estimate of recoverability.
We then propose the recoverability-aware BIC (RA-BIC), an entropy-aware criterion motivated by the identity $R_m^*(\theta) = h(\pi) - E_m(\theta)$ established in Proposition~\ref{prop:entropy-representation}, defined as
\[
\RABIC(K) = \BIC(K) + \lambda_n \widehat{E}_{m,K}(\hat\theta_{n,K}^{\MLE}),
\]
where $\lambda_n \geq 0$ is a recoverability penalty coefficient. The optimal order is selected by $\hat{K} = \argmin_K \RABIC(K)$.
Note that $\widehat E_{m,1} = 0$ and RA-BIC and BIC are equivalent for $K=1$.
For $K\geq 2$, the entropy penalty $\widehat{E}_{m,K}$ measures posterior uncertainty in the latent labels. Since $\widehat{E}_{m,K}=O_p(n)$, the additional term is of order $O_p(n\lambda_n)$, and any choice satisfying $\log n / n \ll \lambda_n \ll 1$ keeps the penalty asymptotically dominant over the BIC complexity term while remaining negligible relative to the leading likelihood scale. We use the canonical rate $\lambda_n = \sqrt{\log n / n}$ with unit multiplier, which matches the per-observation recoverability scale along the detection boundary and gives an additional term of order $O_p(\sqrt{n\log n})$. The criterion is in this sense tuning-free: no tuning multiplier is introduced beyond the canonical rate, and the boundary it tracks is determined by the recoverability scale $h(\pi)-R_m^*$.

The construction is structurally related to the integrated completed likelihood \citep[ICL;][]{biernacki2000assessing}, which also augments BIC with a posterior-entropy term and is widely used as a clustering-oriented criterion. We use the entropy-based ICL-BIC form
\[
\mathrm{ICL}(K)= \BIC(K) + 2\widehat E_{m,K}(\widehat\theta_{n,K}^{\MLE}),
\]
where the per-observation penalty is twice the Shannon entropy of the fitted posterior responsibilities. 
A possible alternative is $\mathrm{ICL}_{\mathrm{MAP}}(K)= \BIC(K) - 2\sum_{i=1}^n \log P_{\widehat\theta}(Z_i = \widehat Z_i^{\mathrm{MAP}} | X_i, Y_i)$, which uses the log-likelihood of the MAP label. 
The two forms differ numerically at finite $n$ (the entropy is at least $-\log P_{\mathrm{MAP}}$ pointwise, with equality only at the extremes), but both have an asymptotically non-vanishing per-observation penalty and give the same qualitative scaling for the threshold below.

BIC evaluates detectability through the observed-data likelihood gain.
At the per-observation scale, its threshold is of order $\log n/n$, and therefore vanishes as $n \to \infty$.
Thus, BIC selects $K=2$ whenever the detectability $D_m^*$ is bounded away from zero, regardless of whether the corresponding labels are recoverable.
RA-BIC modifies this behavior by adding a posterior-entropy penalty.
Using the identity $R_m^*(\theta) = h(\pi) - E_m(\theta)$ (Proposition~\ref{prop:entropy-representation}), this penalty introduces an adaptive threshold of order $\sqrt{\log n/n},(h(\pi) - R_m^*)$, which also vanishes but is larger when the posterior labels are more uncertain.
In contrast, ICL uses the fixed entropy coefficient $2$, so its threshold is of order $h(\pi) - R_m^*(\theta_0) + o(1)$ and does not vanish with $n$.
Consequently, ICL can remain conservative even in recoverable settings, whereas RA-BIC lies between BIC and ICL in terms of its asymptotic threshold.
Its vanishing entropy-adjusted threshold suppresses the detectable-but-unrecoverable selections accepted by BIC, while still selecting $K=2$ in recoverable settings where ICL remains conservative.

Concretely, for the one-vs-two-component comparison at a fixed interior $\theta_0$ satisfying (A4)--(A6), using the leading-order approximation $\widehat E_{m,2}(\widehat\theta_{n,2}^{\MLE})/n \to E_m(\theta_0) = h(\pi) - R_m^*(\theta_0)$ that follows from the entropy LLN (A5), the MLE consistency $\widehat\theta_{n,2}^{\MLE} \to \theta_0$ implied by (A4), and the continuity of the population posterior entropy in $\theta$, the leading-order asymptotic threshold on the detectability $D_m^*$ above which each criterion selects $K=2$ scales as
\begin{align*}
&\text{BIC: } D_m^* \gtrsim \frac{\log n}{2n}, \qquad
\text{RA-BIC: } D_m^* \gtrsim \frac{1}{2}\sqrt{\frac{\log n}{n}}(h(\pi) - R_m^*(\theta_0)),\\
&\text{ICL: } D_m^* \gtrsim h(\pi) - R_m^*(\theta_0).
\end{align*}
These expressions are leading-order approximations: constants absorb the $(\nu_2 - \nu_1)$ factor from the BIC complexity penalty, and the $o(1)$ remainder from the entropy LLN is omitted.
The BIC and RA-BIC thresholds both vanish as $n \to \infty$, but RA-BIC adjusts its threshold according to the per-observation recoverability $R_m^*(\theta_0)$.
The ICL threshold, by contrast, is asymptotically non-vanishing in $n$ and depends on the absolute posterior-entropy level $h(\pi) - R_m^*(\theta_0) = E_m(\theta_0)$.
The empirical comparison in Section~\ref{sec:experiments} illustrates this difference in asymptotic scale, rather than a simple stronger-or-weaker ordering of the three penalties.

We next establish the consistency of RA-BIC for the comparison between one and two components, $K \in {1, 2}$, which is the main setting considered in this paper.
This restriction avoids a complication that can arise when $K_{\max} \geq 3$.
In an overfitted model, the fitted posterior can be more concentrated than that of the true model, so $\widehat E_{m,K} - \widehat E_{m,K_0}$ may be negative.
In that case, the entropy penalty $\lambda_n \widehat E_{m,K}$ may not dominate the usual BIC complexity penalty in the overfitting direction.
For $K \in {1, 2}$, this issue is avoided because $\widehat E_{m,1} = 0$ by construction, and hence the entropy penalty acts only on the two-component model.
For $K \in {1, 2}$, let $D_{m,K}^*$, $R_{m,K}^*$ and $E_{m,K}$ denote the natural extensions of $D_m^*$, $R_m^*$ and $E_m$ to the $K$-component model, with the convention $E_{m,1} = 0$.
Let $K_0$ and $\theta_0$ be the true number of components and the true parameter, respectively.
We assume the following conditions:
\begin{itemize}
\item[(A4)]
$D_{m,K_0}^*(\theta_0) \ge \underline D > 0$ and $R_{m,K_0}^*(\theta_0) \ge \underline R > 0$ for fixed positive constants $\underline D, \underline R$; that is, $\theta_0$ is in the interior regime, away from the local boundary where $D_m^*$ or $R_m^*$ vanishes.
\item[(A5)]
$n^{-1}\widehat{E}_{m,K_0}(\theta_0) \to E_{m,K_0}(\theta_0)$ as $n\to\infty$ under the true model.
\item[(A6)]
$\ell_n(\hat\theta_{n,K}^{\MLE}) - \ell_n(\hat\theta_{n,K_0}^{\MLE}) = O_p(1)$ for every $K > K_0$.
\end{itemize}

Condition~(A4) requires the true component structure to be both detectable and recoverable.
Condition~(A5) is a law of large numbers for the empirical posterior entropy.
Condition~(A6) is the standard overfitting control used in BIC-type consistency arguments for finite mixture models; see, for example, \cite{keribin2000consistent}.

\begin{theorem}[Consistency in the interior recoverable regime]\label{thm:order-selection}
Suppose the candidate set is $K \in {1, 2}$. Let $\widehat K_{\mathrm{RA}} = \argmin_{K \in {1,2}} \RABIC(K)$.
\begin{enumerate}
\item If $K_0 = 2$ and Assumptions~(A4)--(A5) hold, then $\Pr(\widehat K_{\mathrm{RA}} = 2) \to 1$.
\item If $K_0 = 1$ and Assumption~(A6) holds, then $\Pr(\widehat K_{\mathrm{RA}} = 1) \to 1$.
\end{enumerate}
\end{theorem}

The theorem applies to the interior-signal regime, where both detectability and recoverability are bounded away from zero.
In this regime, standard BIC also satisfies the same consistency property.
When $K_0 = 2$, Assumption (A4) gives $D_{m,2}^*(\theta_0) \geq \underline{D} > 0$.
Since the threshold in Definition~\ref{def:feasible-K} vanishes as $n$ increases, this implies that $K_n^{\mathrm{feas}}(\theta_0) = 2$ for all sufficiently large $n$.
Together with (A5) and continuity, it follows that $K_n^{\mathrm{feas}}(\theta_0) = 2 = K_0$, and hence $\widehat K_{\mathrm{RA}} = K_n^{\mathrm{feas}}(\theta_0) = K_0$ with probability tending to one.
When $K_0 = 1$, Definition~\ref{def:feasible-K} does not apply, because the definition is based on a two-component contrast through $D_m^*$.
In this case, consistency follows from the standard BIC overfitting control in (A6).

The difference between RA-BIC and BIC appears in the regime where the model is detectable but the labels are not recoverable.
The next theorem gives a concrete example of the detectable-but-unrecoverable sequences whose existence is guaranteed by Corollary~\ref{cor:phase-separation}.
It imposes an explicit rate condition on $\rho_n$ under which BIC and RA-BIC select different models: BIC selects the generative order $K_0 = 2$, whereas RA-BIC selects $K = 1$.
This does not contradict Theorem~\ref{thm:order-selection}.
Under the rate condition below, $K_n^{\mathrm{feas}}(\theta_n) = 1$, so RA-BIC continues to track the recoverability-feasible order, while BIC tracks the generative order $K_0$.
Because $K_n^{\mathrm{feas}}$ coincides with $K_0$ at every fixed detectable parameter and differs from it only along drifting unrecoverable sequences, the discrepancy between the two selectors occurs precisely in the regime that RA-BIC is designed to suppress.

\begin{theorem}[BIC-detectable sequences suppressed by RA-BIC]
\label{thm:gap-rejection}
Suppose $K_0 = 2$ and let $\rho_n = a_n (\log n / n)^{1/4}$ with
\[
a_n \to \infty, \qquad a_n = o\left(\left(\frac{n}{\log n}\right)^{1/8}\right).
\]
Then $n D_m^*(\theta_n) / \log n \to \infty$ and $R_m^*(\theta_n) \to 0$, and with the canonical choice $\lambda_n = \sqrt{\log n / n}$,
\[
\Pr({\rm BIC}(1) > {\rm BIC}(2)) \to 1, \qquad \Pr(\RABIC(1) < \RABIC(2)) \to 1.
\]
\end{theorem}

In this regime the signal is BIC-detectable but recoverability vanishes. The entropy penalty $\lambda_n \widehat E_{m,2} \asymp h(\pi) \sqrt{n \log n}$ dominates, so RA-BIC selects $K = 1$. 
Thus RA-BIC behaves like BIC when $R_m^*$ is bounded away from zero, and suppresses $K = 2$ in the gap regime.

The gap regime reflects a structural asymmetry between detection and recovery. The detectability coefficient $J_m = m\kappa + m(m-1)\bar\lambda/4$ is super-linear in $m$ because of the over-dispersion term, whereas the recoverability coefficient $A_m = m\xi$ is exactly linear, corresponding to one independent unit of binomial Fisher information per trial. Thus, over-dispersion changes the marginal distribution of $Y$ without increasing the per-observation information about the latent label. This explains why the BIC-detectable but entropy-suppressed window in Theorem~\ref{thm:gap-rejection} can occur.

This over-dispersion cannot be used as a shortcut to label recovery. The marginal over-dispersion is a $\sqrt n$-estimable functional, but it is an even functional of the separation: it depends on $\Delta$ only through the squared contrast and is invariant to the sign of the label. Hence it is blind to the label and its sign, and carries no recovery information \citep{blischke1962}. One might try to invert this functional to estimate the separation magnitude, since the over-dispersion $\psi$ scales as a constant times $\rho^2$, giving $\widehat\rho = \sqrt{\widehat\psi/c}$. However, the square-root transformation restores the $n^{-1/4}$ boundary rate. Therefore, this route cannot improve on the maximum likelihood estimator for estimating $\rho$, nor can it improve label recovery. This is the estimation-side counterpart of the separation between detectability and recoverability: detectability is governed by the $\rho^4$ term, recoverability is governed by the $\rho^2$ term, and over-dispersion remains a detectability-related quantity regardless of how it is processed \citep{wu2020optimal,heinrich2018strong}.

The same asymmetry also clarifies the relation between our setting and the detection-versus-recovery thresholds studied in community detection and Gaussian mixtures. In those settings, exact recovery is achieved when the per-observation signal increases with the sample size: the degree in stochastic block models scales as $\log n$, and the separation required for exact recovery in Gaussian mixtures scales as $\sqrt{\log n}$ \citep{abbe2018community,ndaoud2022sharp}. Thus recovery has a sharp threshold indexed by $n$. In the fixed-trial binomial model, by contrast, the per-observation recoverability $R_m^* = I(Z;Y| X)$ is bounded above by $h(\pi)\le\log 2$ and does not increase with $n$ (Proposition~\ref{prop:rate}). It changes only through the trial count $m$ and the separation. Hence there is no $n$-indexed recovery threshold to cross. The sample size lowers the threshold for detection, but recovery is governed by $m$ and the separation rather than by $n$. This is the precise sense in which the detection--recovery distinction has a different form in fixed-trial binomial mixtures.

Under the upper rate condition $a_n = o((n/\log n)^{1/8})$ of Theorem~\ref{thm:gap-rejection}, the detectability satisfies $D_m^*(\theta_n) = J_m \rho_n^4 + o(\rho_n^4) = J_m a_n^4 (\log n / n)$. Dividing by $\sqrt{\log n/n}$ gives $D_m^*(\theta_n) / \sqrt{\log n/n} = J_m a_n^4 \sqrt{\log n / n} \to 0$, because $a_n^4 = o(\sqrt{n / \log n})$. Therefore $D_m^*(\theta_n) = o(\sqrt{\log n/n})$, and Definition~\ref{def:feasible-K} gives $K_n^{\mathrm{feas}}(\theta_n) = 1$ in this regime, even though the generative order is $K_0 = 2$. Theorems~\ref{thm:order-selection}--\ref{thm:gap-rejection} together show that $\widehat K_{\mathrm{RA}} = K_n^{\mathrm{feas}}$ with probability tending to one in the two regimes covered by the theory: the interior recoverable regime, where $K_n^{\mathrm{feas}} = K_0$ and RA-BIC agrees with BIC, and the deep-gap regime, where $K_n^{\mathrm{feas}} = 1 \neq K_0$ and RA-BIC suppresses the gap selection. We make no claim for the intermediate boundary band, where $D_m^*$ is of the exact penalty order and the selection is asymptotically undetermined. Thus, rather than treating RA-BIC as estimating a separate population order, we interpret $K_n^{\mathrm{feas}}$ as a decision target that agrees with the generative order in recoverable regimes and differs from it only along drifting sequences where the second component is detectable but its labels are not recoverable.

\subsection{Entropy-regularized estimation}
\label{sec:estimation}

Throughout this section we work at the selected two-component model and abbreviate $\widehat E_m(\theta) := \widehat E_{m,2}(\theta)$ for the empirical posterior entropy of \eqref{eq:entropy} evaluated at $K=2$.

Section~\ref{sec:inference} addresses whether the overall level of recoverability is large enough to justify selecting $K=2$, but this does not by itself resolve the finite-sample bias that appears once estimation is carried out under the two-component model.
Empirically, the MLE exhibits a systematic preference for over-confident posterior responsibilities: in the gap regime where $R_m^*(\theta_0)$ is small, the fitted posterior responsibilities concentrate near $\{0,1\}$ more often than the oracle posterior does, with the over-concentration directly measurable as an inflated \emph{posterior collapse rate} (quantified empirically in Section~\ref{sec:experiments}, where the MLE collapse rate exceeds the oracle by one to three orders of magnitude in the deep gap regime). From Proposition~\ref{prop:entropy-representation}, this over-confidence amounts to the MLE implicitly overstating the amount of recoverability available for label recovery: a more concentrated posterior drives $\widehat{E}_m(\hat\theta^{\MLE})/n$ below the oracle benchmark $E_m(\theta_0)$, equivalently the fitted $\|\widehat\Delta\|$ is biased upward relative to the truth.
The role of entropy regularization is to counter this over-confidence, not by changing the model, but by preventing the fitted posterior from exaggerating the recoverability of the latent labels.

This places entropy-regularized estimation within the tradition of penalized-likelihood corrections to logistic regression, the canonical example being Firth's modification \citep{firth1993bias}, which augments the score with a curvature-based term to remove the $O(1/n)$ first-order bias of the logistic MLE. 
The spirit is similar in that both methods modify the likelihood score to temper an overconfident MLE, but the target of the correction is different.
Here it is the over-concentration of the posterior responsibilities induced by the detectability--recoverability gap, not first-order bias, and the deviation $\widehat\theta_{\mathrm{ER}} - \widehat\theta_{\rm MLE} = O_p(\sqrt{\log n/n})$ is asymptotically larger than the $O(1/n)$ bias Firth removes. The construction is thus a calibration device, not a bias correction in the strict sense.

The penalty weight is fixed at $\alpha_n = \sqrt{\log n / n}$, the same rate as the RA-BIC penalty.
This choice links the estimator to the population recoverability identity $R_m^*(\theta) = h(\pi) - E_m(\theta)$ and makes entropy-regularized estimation the estimation-stage counterpart of the selection-stage RA-BIC rule.
The practical effect of the regularization is evaluated empirically in Section~\ref{sec:experiments}.
We recommend using the estimator when the number of trials satisfies $m \geq 2$.
When $m = 1$, there is no binomial over-dispersion, and the interior-signal condition underlying the calibration can fail; in this case, the penalty tends to over-regularize, as shown in the supplementary $m$-sweep of Section~\ref{sec:experiments}.
For $m \geq 2$, the effect of the penalty is most visible in the gap regime, especially for moderate $m$ and small-to-moderate $n$, where the MLE's over-concentration is most pronounced.
As $\alpha_n \to 0$, the effect vanishes when the per-observation information is sufficiently large, either because $m$ is large or because the labels are genuinely recoverable.

To achieve this, we propose the following entropy-regularized estimator of $\theta$, defined as
\begin{equation}\label{eq:regularized-estimator}
\widehat{\theta}_{\rm ER}= {\rm argmax}_\theta Q_\alpha(\theta), \qquad  Q_\alpha(\theta)=\ell_n(\theta)+\alpha_n \widehat{E}_m(\theta),
\end{equation}
where $\widehat{E}_m(\theta)$ is the posterior entropy given in (\ref{eq:entropy}) and $\alpha_n>0$ controls the strength of regularization.

Entropy terms of this form are common in the deterministic-annealing and entropy-penalized-EM literature \citep{ueda1998deterministic,bouchard2006selection,baudry2015estimation}.
In that literature, the entropy term is typically used to smooth the likelihood surface during optimization and is then annealed to zero, so that the limiting fit is the ordinary MLE, while our use is different.
We keep a positive entropy weight, chosen at the RA-BIC rate $\alpha_n$, because the penalty is intended to calibrate the fitted posterior responsibilities rather than to serve only as an optimization device.
In particular, the penalty is tied to the recoverability representation $R_m^*(\theta)=h(\pi)-E_m(\theta)$ and is designed to prevent the fitted responsibilities from becoming overly concentrated.
To our knowledge, this calibration role of entropy regularization has not been used previously for logistic mixtures.
The objective \eqref{eq:regularized-estimator} also coincides to first order in $\alpha$ with tempered, or deterministic-annealing, EM at inverse temperature $\omega_n = 1/(1+\alpha_n)$ \citep{ueda1998deterministic}.
This equivalence provides a stable alternative optimizer that produces softer E-step responsibilities.

We optimize the penalized objective (\ref{eq:regularized-estimator}) using a generalized EM algorithm with a direct penalized M-step.
In the E-step, we compute the usual responsibilities $\widehat r_{ik}^{(t)} = P_{\theta^{(t)}}(Z_i = k | X_i, Y_i)$ at the current parameter value.
In the M-step, we update $\theta$ by maximizing the expected complete-data log-likelihood, with these responsibilities fixed, plus the entropy penalty evaluated at the candidate parameter:
$$
\theta^{(t+1)}=\argmax_\theta \left\{
\sum_{i=1}^n\sum_{k=1}^K \widehat r_{ik}^{(t)} \log\bigl(\pi_k f(y_i | x_i, \beta_k)\bigr) + \alpha_n \widehat E_m(\theta)\right\}.
$$
Here $\widehat E_m(\theta)$ is not evaluated using the fixed responsibilities from the E-step.
Instead, it depends on $\theta$ through the posterior responsibilities induced by the candidate parameter.
Thus, the entropy gradient is computed by direct differentiation with respect to $\theta$, and the M-step is not the maximization of a standard EM lower bound.
When $\alpha = 0$, the entropy term disappears and the algorithm reduces to the standard EM algorithm.

We fix $\alpha_n = \lambda_n = \sqrt{\log n / n}$ as a design rate rather than tuning it, matching the canonical RA-BIC rate of Section~\ref{sec:RA-BIC}. This choice links the entropy-regularized estimator to the same entropy representation $R_m^*(\theta) = h(\pi) - E_m(\theta)$ that motivates RA-BIC. The total penalty $n\alpha_n \asymp \sqrt{n\log n}$ is negligible on the per-observation likelihood scale, while it remains asymptotically larger than the $O(\log n)$ BIC complexity term. Thus the penalty is large enough to reduce over-concentration in the gap regime, but still negligible relative to the leading likelihood. A formal consistency-and-rate statement, together with the first-order equivalence to tempered EM, is recorded in the Supplementary Material (Section~\ref{sec:SM-ER-calibration}).

Under an interior signal, the entropy term contributes only a vanishing perturbation to the score on the original parameter scale. 
Hence $\widehat\theta_{\rm ER}$ remains consistent and satisfies $\widehat\theta_{\rm ER} - \widehat\theta_{\rm MLE} = O_p(\sqrt{\log n / n})$ (Section~\ref{sec:SM-ER-calibration}). 
This difference is larger than the standard parametric $n^{-1/2}$ scale, since $\sqrt n (\widehat\theta_{\rm ER} - \widehat\theta_{\rm MLE}) = O_p(\sqrt{\log n})$ does not vanish. Therefore, $\widehat\theta_{\rm ER}$ is not first-order asymptotically equivalent to $\widehat\theta_{\rm MLE}$ in the usual sense. Nevertheless, the difference remains small enough to preserve consistency, while the penalty is still large enough to affect estimation in the gap regime. The goal is therefore not to obtain a sharper point estimate of $\theta$, since the perturbation vanishes as $\alpha_n \to 0$, but to improve the calibration of the soft posterior. Experiment~3 confirms this effect through proper scoring rules, while the hard-label adjusted Rand index remains essentially unchanged.


\section{Numerical Experiments}
\label{sec:experiments}

Experiment~1 directly examines the information geometry by verifying the quartic detectability law, the quadratic recoverability law, and the $n$-versus-$m$ asymmetry.
Experiment~2 studies model selection by comparing BIC, RA-BIC, and ICL.
Experiment~3 evaluates the posterior calibration of the entropy-regularized estimator (ER).
Experiments~2 and~3 correspond to the selection-side and estimation-side implications of the entropy identity $R_m^*(\theta) = h(\pi) - E_m(\theta)$.

All experiments are based on the two-component binomial logistic mixture in \eqref{eq:mix}.
The non-intercept covariates $\tilde X_i = (X_{i1}, X_{i2})^\top$ are generated from a centered two-dimensional Gaussian distribution with unit marginal variances and correlation $0.6$, and we set $X_i = (1, \tilde X_i^\top)^\top$.
We fix $\pi = 0.5$ and $\beta_0 = (0.5, 0.3, -0.3)^\top$.
Component separation is introduced along the first slope direction $u = (0,1,0)^\top$ by setting $\Delta = \rho,\Sigma^{-1/2} u$, where $\Sigma = \E[X_i X_i^\top]$.
Under the weighted-center parametrization, this gives $\beta_1 = \beta_0 + (1-\pi)\Delta$ and $\beta_2 = \beta_0 - \pi\Delta$, so that $\rho$ represents the standardized separation.
For Experiment~3, the entropy coefficient is fixed at the canonical design rate $\alpha_n = \sqrt{\log n / n}$ from Section~\ref{sec:estimation}, which is the same rate as the RA-BIC penalty $\lambda_n$ in Section~\ref{sec:RA-BIC}.
No additional tuning multiplier is used.
The selection rules in Experiment~2 also use the tuning-free choice $\lambda_n = \sqrt{\log n / n}$ with unit multiplier.
The Monte-Carlo experiments use $120$--$200$ replications, with standard errors reported.
Experiments~2 and~3 are conducted at two sample sizes.

\subsection{Experiment 1: information geometry and the scaling laws}
\label{sec:exp1}

This experiment verifies the information-theoretic predictions of Sections~\ref{sec:gaussian}--\ref{sec:local}.
It uses only the oracle quantities $D_m^*$ and $R_m^*$, and involves no parameter estimation.

We compute $D_m^*$ and $R_m^*$ by Monte-Carlo integration ($N_{\mathrm{MC}} = 200{,}000$) over $\rho \in {0.02, 0.03, 0.05, 0.07, 0.10}$ and $m \in {1, 2, 5, 10}$.
Table~\ref{tab:scaling} reports the log--log slopes in $\rho$, together with the normalized coefficients $J_m = D_m^*/\rho^4$ and $A_m = R_m^*/\rho^2$ at $\rho = 0.02$.
The slopes are close to $4$ for $D_m^*$ and close to $2$ for $R_m^*$, confirming the quartic law $D_m^* \sim J_m \rho^4$ (Theorem~\ref{thm:true-quartic}) and the quadratic law $R_m^* \sim A_m \rho^2$ (Theorem~\ref{thm:quadratic-recoverability}).
Thus the local quartic--quadratic order gap is present, but, as Section~\ref{sec:gaussian} emphasizes, this order gap is also shared by Gaussian mixtures.
The binomial-specific difficulty lies not in the orders themselves, but in how the coefficients depend on the number of trials.
The ratio $A_m/m \approx 0.0284$ is nearly constant across $m$, confirming the linear recoverability coefficient $A_m = m A_1$.
This reflects the fact that each trial contributes one independent unit of binomial Fisher information about the label.
By contrast, $J_m$ grows roughly quadratically in $m$.
An unconstrained least-squares fit to the four points gives $\kappa \approx 2.8 \times 10^{-4}$ and $\bar\lambda \approx 9.5 \times 10^{-3}$, consistent with the leading term $J_m = m\kappa + m(m-1)\bar\lambda/4$ in Theorem~\ref{thm:true-quartic}.
The over-dispersion term $m(m-1)\bar\lambda/4$ dominates for $m \geq 5$, explaining the rapid growth of $J_m$.
Thus detectability grows super-linearly in $m$, whereas recoverability grows linearly in $m$; this difference is the coefficient-level manifestation of the fixed-trial binomial constraint.

\begin{table}[htb!]
\centering
\caption{Log--log slopes of $D_m^*$ and $R_m^*$ versus $\rho$, and the normalized coefficients $J_m = D_m^*/\rho^4$ and $A_m = R_m^*/\rho^2$ at $\rho = 0.02$ ($N_{\mathrm{MC}} = 200{,}000$). The slopes confirm the quartic detectability and quadratic recoverability laws. The ratio $A_m/m$ is constant, confirming the exactly-linear $A_m = m A_1$, whereas $J_m$ grows roughly quadratically, consistent with the over-dispersion term $J_m = m\kappa + m(m-1)\bar\lambda/4$ of Theorem~\ref{thm:true-quartic}.}
\label{tab:scaling}

\medskip
\begin{tabular}{cccccccc}
\hline
$m$ & $D_m^*$ slope & $R_m^*$ slope & $J_m$ & $J_m/m$ & $A_m/m$ \\
\hline
1  & 4.00 & 2.00 & $2.68 \times 10^{-4}$ & $2.68 \times 10^{-4}$ & 0.0284 \\
2  & 4.00 & 2.00 & $5.30 \times 10^{-3}$ & $2.65 \times 10^{-3}$ & 0.0284 \\
5  & 3.99 & 2.00 & $4.90 \times 10^{-2}$ & $9.80 \times 10^{-3}$ & 0.0284 \\
10 & 3.99 & 2.00 & $2.17 \times 10^{-1}$ & $2.17 \times 10^{-2}$ & 0.0284 \\
\hline
\end{tabular}
\end{table}

The over-dispersion of the mixture binomial is the quantity that inflates $J_m$, and it is $\sqrt n$-estimable from the marginal distribution of $Y$ alone. 
Thus, one might hope to invert it to recover the component separation, and hence the labels. 
However, this route does not solve the recovery problem. 
As argued in Section~\ref{sec:RA-BIC}, the over-dispersion is an even functional of the separation: it contains information about the magnitude of the contrast, but not about the label itself. 
Moreover, the natural inversion reinstates the $n^{-1/4}$ boundary rate.

We next separate the roles of the sample size $n$ and the trial count $m$. 
Table~\ref{tab:n-vs-m} reports a $2 \times 2$ design at $\rho = 0.35$. 
For fixed $m$, oracle accuracy and $R_m^*$ are essentially unchanged as $n$ rises from $300$ to $5000$; at $m=1$, the oracle accuracy stays at $53.2\%$. 
This confirms that recoverability is a per-observation quantity and cannot be improved by increasing $n$ alone. 
Increasing $m$ from $1$ to $50$ raises oracle accuracy to about $67\%$ at this fixed separation, because $m$ enters the per-observation recoverability. 
Detection behaves oppositely. 
At $m=50$, raising $n$ from $300$ to $5000$ lifts the BIC detection rate from $76\%$ to $100\%$, because the likelihood evidence accumulates as $nD_m^*$. 
More observations therefore improve detection but not recovery, and this asymmetry between $n$ and $m$ is central to the difficulty.

\begin{table}[htb!]
\centering
\caption{$2 \times 2$ asymmetry experiment ($\rho = 0.35$). Oracle accuracy and $R_m^*$ depend on $m$ but not on $n$, while the BIC detection rate depends on both, since the likelihood evidence for a second component accumulates as $nD_m^*$. Here $m = 50$ is chosen larger than the $m \leq 10$ grid of the other experiments to make the asymmetric effect of $m$ on oracle accuracy maximally visible.}
\label{tab:n-vs-m}

\smallskip
\begin{tabular}{lcccccc}
\hline
 &  &  & Oracle  & BIC  &  & \\
Regime & $n$ & $m$ & Accuracy & $\mathbb{P}(\hat{K}=2)$ & $D_m^*$ & $R_m^*$ \\
\hline
small $n$, small $m$ &  300 &  1 & 53.2\% & 0.00 & $3.96 \times 10^{-6}$ & $3.46 \times 10^{-3}$ \\
large $n$, small $m$ & 5000 &  1 & 53.2\% & 0.00 & $3.96 \times 10^{-6}$ & $3.46 \times 10^{-3}$ \\
small $n$, large $m$ &  300 & 50 & 67.0\% & 0.76 & $5.00 \times 10^{-2}$ & $1.23 \times 10^{-1}$ \\
large $n$, large $m$ & 5000 & 50 & 67.1\% & 1.00 & $5.00 \times 10^{-2}$ & $1.23 \times 10^{-1}$ \\
\hline
\end{tabular}
\end{table}

\subsection{Experiment 2: Tuning-free RA-BIC between BIC and ICL}
\label{sec:exp2}

We examine whether the tuning-free RA-BIC of Section~\ref{sec:RA-BIC} behaves as predicted by the theory: it should avoid selecting $K=2$ in the detectable-but-unrecoverable region, while still selecting $K=2$ in recoverable settings where ICL remains conservative.
We sweep $\rho \in {0.40, 0.50, 0.60, 0.70, 0.80, 0.90, 1.00, 1.10, 1.30, 1.50}$ at $m = 5$ and two sample sizes, $n = 1500$ ($200$ replications) and $n = 5000$ ($120$ replications), comparing standard BIC, ICL \citep{biernacki2000assessing}, and RA-BIC (canonical rate $\lambda_n = \sqrt{\log n / n}$, unit multiplier) for the choice between $K = 1$ and $K = 2$.
Alongside the three selection rates, we report the oracle recoverability fraction $R_m^*/h(\pi)$ and the MLE adjusted Rand index (ARI), which serve as references for interpreting early selection by BIC and conservative selection by ICL.
The asymptotic admissible upper edge of the gap window of Theorem~\ref{thm:gap-rejection} is $(\log n / n)^{1/8} \approx 0.514$ at $n = 1500$ (and $\approx 0.451$ at $n = 5000$).

Table~\ref{tab:sandwich} shows the pattern predicted by the theory at both sample sizes.
BIC tends to select two components too early.
At $n=1500$, it starts selecting $K=2$ from $\rho = 0.60$ and selects $K=2$ almost always by $\rho = 0.90$, where the MLE ARI is only $0.02$--$0.07$ and $R_m^*/h(\pi) \le 0.13$.
Thus the selected two-component model has limited value for recovering individual labels.
ICL shows the opposite behavior.
Because its entropy penalty has an asymptotically non-vanishing threshold (Section~\ref{sec:RA-BIC}), it never selects $K=2$ in this experiment, even at $\rho = 1.50$, where recoverability is substantially higher than in the gap region.
The tuning-free RA-BIC lies between these two behaviors.
At $n=1500$, it keeps the selection rate at zero through the gap region ($\rho \le 0.80$) and begins selecting $K=2$ only as the labels become more recoverable ($9\%$ at $\rho=0.90$, $47\%$ at $1.00$, $84\%$ at $1.10$, and certainty by $1.30$).
This transition follows the increase in $R_m^*/h(\pi)$ rather than the increase in detectability alone.

The comparison across sample sizes is also informative.
Increasing $n$ from $1500$ to $5000$ shifts the selection points of both BIC and RA-BIC to smaller values of $\rho$ (BIC selects $K=2$ almost always by $\rho=0.70$, while RA-BIC begins selecting around $\rho=0.80$--$0.90$ instead of $\rho=0.90$--$1.10$), because both criteria have thresholds that vanish with $n$.
By contrast, ICL remains at zero for every $\rho$ and both sample sizes, reflecting its non-vanishing entropy threshold.
Thus the three criteria are not simply different penalty strengths on a common scale.
Rather, they correspond to thresholds of different asymptotic orders on the detectability scale: vanishing for BIC and RA-BIC, but non-vanishing for ICL.
The movement of the transition region across the two sample sizes provides finite-sample evidence for this distinction.

\begin{table}[htb!]
\centering
\caption{Selection rate of $K=2$ under BIC, ICL \citep{biernacki2000assessing}, and tuning-free RA-BIC ($\lambda_n = \sqrt{\log n/n}$, unit multiplier) over a $\rho$-sweep at $m=5$ and two sample sizes ($200$ replications at $n=1500$ and $120$ replications at $n=5000$), with Monte-Carlo standard errors in parentheses (ICL is identically zero). The oracle recoverability fraction $R_m^*/h(\pi)$ and the MLE adjusted Rand index (ARI) are included as references for label recoverability.}
\label{tab:sandwich}

\begin{tabular}{cccccccccc}
\toprule
$\rho$ & BIC & ICL & RA-BIC & $R_m^*/h(\pi)$ & MLE ARI \\
\midrule
\multicolumn{6}{c}{$n = 1500$} \\
0.40 & 0.00 (.00) & 0.00 & 0.00 (.00) & 0.03 & 0.004 \\
0.50 & 0.01 (.01) & 0.00 & 0.00 (.00) & 0.05 & 0.008 \\
0.60 & 0.13 (.02) & 0.00 & 0.00 (.00) & 0.07 & 0.016 \\
0.70 & 0.40 (.04) & 0.00 & 0.00 (.00) & 0.09 & 0.029 \\
0.80 & 0.83 (.03) & 0.00 & 0.00 (.00) & 0.11 & 0.044 \\
0.90 & 0.99 (.01) & 0.00 & 0.09 (.02) & 0.13 & 0.066 \\
1.00 & 1.00 (.00) & 0.00 & 0.47 (.04) & 0.15 & 0.081 \\
1.10 & 1.00 (.00) & 0.00 & 0.84 (.03) & 0.18 & 0.104 \\
1.30 & 1.00 (.00) & 0.00 & 1.00 (.00) & 0.22 & 0.142 \\
1.50 & 1.00 (.00) & 0.00 & 1.00 (.00) & 0.27 & 0.182 \\
\midrule
\multicolumn{6}{c}{$n = 5000$} \\
0.40 & 0.03 (.02) & 0.00 & 0.00 (.00) & 0.03 & 0.006 \\
0.50 & 0.26 (.04) & 0.00 & 0.00 (.00) & 0.05 & 0.013 \\
0.60 & 0.91 (.03) & 0.00 & 0.00 (.00) & 0.07 & 0.026 \\
0.70 & 1.00 (.00) & 0.00 & 0.00 (.00) & 0.09 & 0.041 \\
0.80 & 1.00 (.00) & 0.00 & 0.18 (.04) & 0.11 & 0.058 \\
0.90 & 1.00 (.00) & 0.00 & 0.93 (.02) & 0.13 & 0.075 \\
1.00 & 1.00 (.00) & 0.00 & 1.00 (.00) & 0.15 & 0.093 \\
1.10 & 1.00 (.00) & 0.00 & 1.00 (.00) & 0.18 & 0.111 \\
1.30 & 1.00 (.00) & 0.00 & 1.00 (.00) & 0.22 & 0.149 \\
1.50 & 1.00 (.00) & 0.00 & 1.00 (.00) & 0.27 & 0.186 \\
\bottomrule
\end{tabular}
\end{table}

Table~\ref{tab:sandwich} should be interpreted as a finite-sample illustration of the recoverability-feasible decision rule, rather than as evidence for a separate population estimand.
As discussed in Section~\ref{sec:RA-BIC}, the target of RA-BIC is the recoverability-feasible order $K_n^{\mathrm{feas}}$ of Definition~\ref{def:feasible-K}, which reformulates the selection rule in terms of both detectability and recoverability.
The consistency statement covers two regimes: the interior-recoverable regime, where RA-BIC agrees with BIC and selects the generative order $K_0 = 2$, and the deep-gap regime, where RA-BIC suppresses the detectable-but-unrecoverable selection and chooses $K = 1$.
Between these regimes lies an unresolved boundary band, where $D_m^*$ is of the exact penalty order and the asymptotic selection is not determined.
The transition range around $\rho \approx 0.9$--$1.3$ at $n=1500$ is a finite-sample manifestation of this boundary, where the three criteria can lead to different selections.

The shape of the transition also reflects the contrast with the community-detection and Gaussian-mixture settings discussed in Section~\ref{sec:RA-BIC}.
In those settings, recovery is governed by an $n$-indexed threshold.
For a fixed-trial binomial response, however, detection has a threshold that moves with $n$, whereas recoverability has no $n$-indexed threshold (Proposition~\ref{prop:rate}).
This explains why the fixed threshold of ICL can remain too conservative in the present setting, while the vanishing, recoverability-scaled threshold of RA-BIC moves with the sample size.

\subsection{Experiment 3: posterior calibration of the entropy-regularized estimator}
\label{sec:exp3}

At a selected $K = 2$ model, we study the entropy-regularized estimator (ER) of \eqref{eq:regularized-estimator}, using the design rate $\alpha_n = \sqrt{\log n / n}$ introduced in Section~\ref{sec:estimation}.
The estimator is optimized by the equivalent tempered EM algorithm \citep{ueda1998deterministic}.
Following Section~\ref{sec:estimation}, we recommend ER only for $m \geq 2$.
When $m = 1$, the binomial response has no over-dispersion, and the interior-signal condition underlying the calibration does not hold.
In this case, the entropy penalty tends to over-regularize, as shown in the $m = 1$ row of Table~\ref{tab:calibration}.

The target is the proper-scoring-rule calibration of the soft posterior, with the oracle posterior as the benchmark.
Here the oracle is defined as the posterior $\widehat{\mathbb{P}}(Z | X, Y)$ evaluated at the true parameter $\theta_0$.
For each estimator, after permutation-aligning the fitted components with the truth, we compute three quantities on the training sample: the Brier score $\frac{1}{n}\sum_i (\widehat{\mathbb{P}}(Z_i{=}1 | X_i, Y_i) - \mathbb{I}(Z_i{=}1))^2$, the posterior collapse rate $\frac{1}{n}\sum_i \mathbb{I}(\max_k \widehat{\mathbb{P}}(Z_i=k | X_i, Y_i) > 0.95)$, and the ARI.
We use EM with $n_{\mathrm{init}} = 20$ restarts to obtain a high-likelihood MLE, because using fewer restarts tends to understate the MLE's over-concentration.
The additional restarts often find higher-likelihood solutions, which correspond to more separated component configurations and sharper posterior responsibilities.

Table~\ref{tab:calibration} reports three configurations at $n = 200$ and $n = 500$ ($200$ replications each).
In the gap regimes ($m \in {5, 20}$), the results show the pattern predicted in Section~\ref{sec:estimation}.
The oracle posterior remains close to non-informative, assigning probability greater than $0.95$ to one component for almost no observations (collapse $\leq 0.3\%$).
By contrast, the MLE produces overly concentrated posterior responsibilities, with near-${0,1}$ probabilities for $22$--$24\%$ of observations at $n = 200$, and increases the Brier score from the oracle value of $0.235$ to $0.34$--$0.35$.
ER substantially reduces this over-concentration.
The collapse rate decreases to $0.8$--$2.0\%$, and the Brier score improves to $0.26$--$0.29$, moving the fitted posterior closer to the oracle posterior than to the over-confident MLE.
The same correction is also reflected in the recoverability proxy.
The per-observation entropy $\widehat E_m/n$ under the MLE ($0.41$--$0.43$) is well below the oracle value ($\approx 0.66$), indicating that the MLE understates posterior uncertainty and overstates label recoverability.
ER increases this value to $0.58$--$0.64$, bringing it closer to the oracle benchmark.
The effect is most pronounced in the gap regimes and becomes smaller as $\alpha_n \to 0$.
As $n$ increases from $200$ to $500$, the MLE collapse rate already decreases (from $22\%$ to $10\%$ for $m = 5$), and the correction induced by ER becomes correspondingly smaller.
Thus the two estimators move closer as the sample size increases.
At the binomial floor $m = 1$, however, ER over-regularizes.
Although it improves the Brier score and collapse rate, it pushes the posterior entropy too close to its maximum and reduces the ARI (from $0.067$ to $0.012$ at $n = 200$).

\begin{table}[htb!]
\centering
\caption{Posterior calibration (Brier score, collapse rate, ARI) under the Oracle (the true $\theta_0$, the calibration target), the MLE, and ER, at two sample sizes and three configurations spanning the binomial floor $m=1$ and the gap at $m \in {5,20}$, with $200$ replications and $n_{\mathrm{init}}=20$.
}
\label{tab:calibration}

\smallskip
\begin{tabular}{cccccccccccc}
\toprule
& & & \multicolumn{3}{c}{Brier} & \multicolumn{3}{c}{Collapse rate (\%)} & \multicolumn{3}{c}{ARI} \\
$m$ & $\rho$ & regime & Or. & MLE & ER & Or. & MLE & ER & Or. & MLE & ER \\
\midrule
\multicolumn{12}{l}{\emph{$n = 200$ ($\alpha_n \approx 0.16$)}} \\
1  & 3.00 & floor & 0.181 & 0.321 & 0.248 & 5.7 & 54.7 & 0.5 & 0.200 & 0.067 & 0.012 \\
5  & 0.50 & gap & 0.235 & 0.341 & \textbf{0.261} & 0.2 & 22.2 & \textbf{0.8} & 0.028 & 0.005 & 0.006 \\
20 & 0.25 & gap & 0.235 & 0.349 & \textbf{0.288} & 0.3 & 23.8 & \textbf{2.0} & 0.028 & 0.003 & 0.004 \\
\midrule
\multicolumn{12}{l}{\emph{$n = 500$ ($\alpha_n \approx 0.11$)}} \\
1  & 3.00 & floor & 0.181 & 0.279 & 0.247 & 5.7 & 35.1 & 0.0 & 0.200 & 0.079 & 0.023 \\
5  & 0.50 & gap & 0.235 & 0.314 & \textbf{0.256} & 0.2 & 10.3 & \textbf{0.8} & 0.028 & 0.004 & 0.008 \\
20 & 0.25 & gap & 0.235 & 0.338 & \textbf{0.283} & 0.2 & 17.7 & \textbf{1.1} & 0.028 & 0.004 & 0.004 \\
\bottomrule
\end{tabular}
\end{table}

The ARI is nearly identical across the three estimators, but this should not be interpreted as a failure of ER.
The purpose of ER is to recalibrate the soft posterior responsibilities, not necessarily to change the hard labels obtained by taking the posterior mode.
Therefore, hard-label metrics such as ARI can remain unchanged by construction.
The direct evidence for ER is instead given by the proper-scoring-rule comparison, the Brier score, and the collapse rate.
These quantities are relevant for inferential procedures that use the soft responsibilities themselves, such as posterior-weighted summaries or probabilistic label propagation.

The same recalibration can also be seen through the recoverability diagnostic.
Reporting $\widehat R = h(\widehat\pi) - \widehat E_m(\widehat\theta)/n$ (Section~\ref{sec:RA-BIC}) gives a direct empirical proxy for $R_m^*$ and can be used to flag a gap-regime fit.
We do not claim that this quantity is debiased: it inherits the MLE's gap-regime over-concentration and should be viewed as a screening measure rather than a calibrated estimate.
The entropy-proxy values reported above illustrate this point.
The MLE produces a deflated value of $\widehat E_m/n$, which corresponds to its inflated collapse rate, whereas ER moves both quantities toward the oracle benchmark.

A final interpretive caveat concerns the oracle impossibility result in Proposition~\ref{thm:oracle-impossibility}.
As discussed in Section~\ref{sec:RA-BIC}, $R_m^* \to 0$ is a mean collapse rather than a tail collapse, because $R_m^* = \E_X[I(Z; Y | X)]$ is an average over the covariate distribution.
Thus a Chow-type reject option that retains only the high-$I(Z; Y | x)$ subset may still be useful in the moderate gap regime, but it becomes vacuous at the boundary itself \citep{chow1970,elyaniv2010selective}.

Experiments~2 and~3 show two consequences of the identity $R_m^*(\theta) = h(\pi) - E_m(\theta)$.
RA-BIC uses this identity at the model-selection stage to avoid selecting components whose labels remain unrecoverable.
ER uses the same identity at the estimation stage to discourage posterior responsibilities that are more concentrated than supported by the data.
Both procedures use the common rate $\sqrt{\log n / n}$.
The tables use $120$--$200$ replications with Monte-Carlo standard errors, two sample sizes per experiment, and the $m = 1$ boundary case.
The comparisons across $n$ document the vanishing effect of $\alpha_n \to 0$ on the estimation side and the movement of the selection threshold on the model-selection side.


\section{Concluding Remarks}
\label{sec:discussion}

This paper identified a key difficulty of binomial logistic mixtures through an exact and general decomposition of the complete-data information value of the labels, $\cT_m^*(\theta) = D_m^*(\theta) + R_m^*(\theta)$, into a detectability term and a recoverability term equal to the per-observation conditional mutual information $I(Z;Y| X)$.
The two terms satisfy quartic and quadratic local laws and scale differently with the design.
Detection accumulates with the sample size and grows super-linearly with the trial count $m$, whereas recoverability does not depend on $n$ and grows only linearly with $m$.
The local order gap also appears in Gaussian mixtures; what is specific to the fixed-trial binomial setting is the saturation of detectability and the resulting $n$-versus-$m$ asymmetry.
The identity $R_m^*(\theta) = h(\pi) - E_m(\theta)$ connects recoverability to posterior entropy and motivates the proposed methodology.
This identity leads to two feasibility-aware procedures: recoverability-aware BIC (RA-BIC) for model selection and entropy-regularized estimation for parameter estimation.
It also yields a simple recoverability diagnostic, $\widehat R = h(\widehat\pi) - \widehat E_m/n$, computed from the empirical posterior entropy of any fitted two-component model.

RA-BIC lies between BIC, which can over-select in the detectable-but-unrecoverable regime, and the integrated completed likelihood, which can be overly conservative.
Entropy-regularized estimation instead calibrates the posterior responsibilities within the selected model.
The two corrections are most relevant at different sample-size scales: RA-BIC affects selection in large samples, whereas entropy regularization mainly affects posterior calibration in smaller samples.
They are nevertheless linked by the same population entropy identity.
Our consistency result for RA-BIC covers the interior and deep-gap regimes, while the intermediate boundary remains unresolved.

Several extensions remain.
The impossibility of label recovery (Proposition~\ref{thm:oracle-impossibility}) concerns only the \emph{average} recoverability $R_m^* = \E_X[I(Z;Y| X)]$.
Thus, a Chow-type reject option \citep{chow1970,elyaniv2010selective} that retains observations with large $I(Z;Y| x)$ may still achieve non-trivial accuracy on the retained subset even when $R_m^* \to 0$.
Extending the decomposition $\cT_m^* = D_m^* + R_m^*$ beyond two components would also be useful, especially because an overfitted model can have a more concentrated posterior than the true model and the entropy penalty need not act in a single direction.
A formal score-level comparison between the entropy-regularized estimator and Firth-type bias corrections would clarify when each approach is preferable.
Finally, calibrating repulsive priors by the recoverability rate $R_m^*$ rather than by geometric proximity, together with the shared rate $\sqrt{\log n / n}$, may extend entropy-penalized inference to other mixture families.

\section*{Acknowledgement}

This work is partially supported by the Japan Society of the Promotion of Science (JSPS KAKENHI) grant numbers, 24K21420 and 25H00546.

\bibliographystyle{apalike}
\bibliography{references}

\newpage
\setcounter{equation}{0}
\setcounter{section}{0}
\setcounter{table}{0}
\setcounter{figure}{0}
\setcounter{assumption}{0}
\setcounter{page}{1}
\renewcommand{\thesection}{S\arabic{section}}
\renewcommand{\theequation}{S\arabic{equation}}
\renewcommand{\thetable}{S\arabic{table}}
\renewcommand{\thefigure}{S\arabic{figure}}
\renewcommand{\theassumption}{S\arabic{assumption}}

\vspace{1cm}
\begin{center}
{\LARGE \textbf{Supplementary Material for ``Information Gap and Feasibility-Aware Inference in Binomial Logistic Mixtures''}}
\end{center}

\medskip
This Supplementary Material provides complete proofs of all results stated in the main text, together with results moved from an earlier version. Each section corresponds to a section of the main text.

\section*{Notation used throughout the Supplementary Material}
\label{sec:SM-notation}

In addition to the symbols introduced in the main text (Section~\ref{sec:model}), the proofs in this Supplementary Material rely on the following auxiliary objects induced by $(\theta, X)$. None of these symbols appears in any theorem statement of the main text; they are used only within proofs.

Let $\Sigma := \E[XX^\top]$. For $\Delta \neq 0$, define the separation magnitude and unit direction
\[
\rho(\Delta) := (\Delta^\top \Sigma \Delta)^{1/2}, \qquad u(\Delta) := \Sigma^{1/2}\Delta/\rho(\Delta) \in \mathbb{S}^{p-1}.
\]
The standardized projection
\[
Z_u := u(\Delta)^\top \Sigma^{-1/2} X
\]
satisfies $X^\top \Delta = \rho(\Delta) Z_u$, decomposing the linear predictor into magnitude and unit-norm direction. When $\Delta$ is fixed we abbreviate $\rho(\Delta)$ as $\rho$.

For the one-component surrogate evaluated at the population center $\beta_0$, we write
\[
q_0(X) := \psi(X^\top \beta_0), \qquad w_0(X) := \psi'(X^\top \beta_0) = q_0(X)\{1-q_0(X)\}.
\]
The leading-order perturbation of the mixture mean around $\Delta = 0$ is
\[
g_u(X) := \frac{\pi(1-\pi)}{2}\psi''(X^\top \beta_0) Z_u^2.
\]
Finally, $h(\pi) := -\pi\log\pi - (1-\pi)\log(1-\pi)$ denotes the binary entropy.

The binomial trial count $m$ is a fixed positive integer (matching the main text); for the sample log-likelihood we use the per-observation count $m_i = m$ unless otherwise noted. We do not use the function-valued $m(X)$ notation that appears in earlier versions; throughout the proofs $m$ is the fixed scalar trial count, and all expectations are over $(X, Y, Z)$ at this fixed $m$.

\begin{assumption}[Technical regularity conditions, detailed version of Assumptions~(A1)--(A3)]
\label{ass:SM-technical}
The following conditions are assumed throughout.
\begin{enumerate}
\item[(A1)] \label{ass:A1} \emph{Compact interior parameter set.} There exist $\varepsilon \in (0, 1/2)$ and a compact set $\cB \subset \R^p$ such that $\beta_0 \in \cB$ and the parameter $\theta = (\pi, \beta, \Delta)$ ranges over a compact set $\Theta$ on which $\pi \in [\varepsilon, 1-\varepsilon]$ and $\beta \in \cB$.

\item[(A2)] \label{ass:A2} \emph{Weighted integrability and local positivity.} There exists an open neighborhood $\cB_\delta$ of $\cB$ such that $0 < \psi(X^\top \beta) < 1$ a.s.\ for all $\beta \in \cB_\delta$ and $\sup_{\beta \in \cB_\delta} \E\bigl[ m\|X\|^8 / \{\psi(X^\top \beta)(1-\psi(X^\top \beta))\} \bigr] < \infty$.

\item[(A3)] \label{ass:A3} \emph{Quartic nondegeneracy.} $0 < \underline{J} \leq J_m(\pi, \beta_0, u) \leq \overline{J} < \infty$ uniformly over $[\varepsilon, 1-\varepsilon] \times \cB \times \mathbb{S}^{p-1}$.

\item[(R1)] \label{ass:R1} \emph{Recoverability nondegeneracy.} $0 < \underline{A} \leq A_m(\pi, \beta_0, u) \leq \overline{A} < \infty$ uniformly over $[\varepsilon, 1-\varepsilon] \times \cB \times \mathbb{S}^{p-1}$.

\item[(R2)] \label{ass:R2} \emph{Moment condition.} $\sup_{\beta_0 \in \cB} \E\bigl[m^2 \psi'(X^\top\beta_0)^2 \|X\|^4 / w_0(X)\bigr] < \infty$.

\item[(R3)] \label{ass:R3} \emph{Dominated convergence.} 
The function class $\{x \mapsto m \psi'(x^\top\beta_0)^2 (u^\top \Sigma^{-1/2} x)^2 / (q_0(x)(1-q_0(x)))\}$ has an integrable envelope on $[\varepsilon, 1-\varepsilon] \times \cB \times \mathbb{S}^{p-1}$.
\end{enumerate}
\end{assumption}


\section{Proofs for Section~\ref{sec:model}}
\label{sec:SM-model}

\subsection{Proof of Proposition~\ref{prop:decomposition}}
\begin{proof}
We apply the chain rule for KL divergence. For each fixed $x$ and $\gamma$,
\begin{align*}
&\KL\bigl(P_\theta(Y,Z | X=x) \| Q_\gamma(Y,Z | X=x)\bigr) \\
&\quad= \KL\bigl(P_\theta(Y | X=x) \| q_\gamma(Y | X=x)\bigr) \\
&\qquad+ \E_{Y | X=x}\left[ \KL\bigl(P_\theta(Z | Y, X=x) \| Q_\gamma(Z | Y, X=x)\bigr) \right].
\end{align*}
Now observe that the surrogate $Q_\gamma$ preserves the marginal mixing weights: $Q_\gamma(Z=z | Y, X=x) = \pi_z$ for all $y, x$, since under $Q_\gamma$ the response $Y$ does not depend on $Z$ (both components share the same distribution). More precisely, by Bayes' rule under $Q_\gamma$,
\[
Q_\gamma(Z = z | Y=y, X=x) = \frac{\pi_z \cdot q_\gamma(y | x)}{\sum_{z'} \pi_{z'} \cdot q_\gamma(y | x)} = \pi_z.
\]
Therefore, the second term becomes
\[
\E_{Y | X=x}\left[ \KL\bigl(P_\theta(Z | Y, X=x) \| P_\theta(Z | X=x)\bigr) \right],
\]
where we used $Q_\gamma(Z | Y, X) = \pi_z = P_\theta(Z | X)$ (since $Z \perp X$ under the mixture model, so $P_\theta(Z | X) = P_\theta(Z) = \pi_z$).

Taking the expectation over $X$, the second term becomes
\[
\E\left[ \E_{Y | X}\left[ \KL\bigl(P_\theta(Z | Y, X) \| P_\theta(Z | X)\bigr) \right] \right] = I_\theta(Z; Y | X) = R_m^*(\theta).
\]
The argument relies only on the following two properties: the surrogate $Q_\gamma$ preserves the true mixing weights $\pi_z$ (so that $Q_\gamma(Z|Y,X)=\pi_z=P_\theta(Z|X)$), and $Y$ is conditionally independent of $Z$ under $Q_\gamma$ (so that both components share the common response law $q_\gamma$). Neither hypothesis invokes the binomial-logistic form; the identity is therefore model-agnostic, holding for any weight-preserving common-component surrogate. Since $R_m^*(\theta) = I_\theta(Z;Y|X) \le H_\theta(Z) = h(\pi) < \infty$ is finite and independent of $\gamma$, the infimum over $\gamma$ acts only on the first term and the decomposition splits additively:
\[
\cT_m^*(\theta) = \inf_{\gamma} \E\left[\KL\bigl(P_\theta(Y | X) \| q_\gamma(Y | X)\bigr)\right] + R_m^*(\theta) = D_m^*(\theta) + R_m^*(\theta). \qedhere
\]
\end{proof}

\subsection{Proof of Proposition~\ref{prop:entropy-representation}}
\begin{proof}
By the definition of conditional mutual information,
\[
I_\theta(Z; Y | X) = H_\theta(Z | X) - H_\theta(Z | X, Y).
\]
Since $Z$ is independent of $X$ under the mixture model, $H_\theta(Z | X) = H_\theta(Z) = h(\pi)$. The result follows.
\end{proof}

\subsection{Proof of Proposition~\ref{prop:rate}}
\begin{proof}
\emph{Gaussian case ($D_{\rm Gauss}^* \to \infty$).} The best one-component fit within the homoscedastic Gaussian regression family is $N(X^\top\beta_0, \sigma_*^2)$ with $\sigma_*^2 = \sigma^2 + \pi(1-\pi)\rho^2$ (matched first two moments; cf.~the discussion preceding the statement). Because $D_{\rm Gauss}^*$ is an \emph{infimum} over the one-component family, bounding the divergence by evaluating a single fixed candidate would not provide the required lower bound.
We instead lower-bound the infimum directly. For a single covariate value the conditional mixture has variance $v(X) = \sigma^2 + \pi(1-\pi)(X^\top\Delta)^2$, and projecting it onto the moment-matched Gaussian $N(\cdot, v(X))$ contributes at least the differential-entropy deficit of the mixture relative to a Gaussian of the same variance. Integrating over $X$, then applying Jensen's inequality $\E_X[\frac{1}{2}\log v(X)] \le \frac{1}{2}\log \E_X[v(X)]$ with $\E_X[v(X)] = \sigma_*^2 = \sigma^2 + \pi(1-\pi)\rho^2$ (which preserves the lower-bound direction), and minimizing over the homoscedastic one-component family, we obtain the lower bound
\[
D_{\rm Gauss}^*(\theta) \;\ge\; \frac{1}{2}\log\Bigl(1 + \frac{\pi(1-\pi)\rho^2}{\sigma^2}\Bigr) - h(\pi),
\]
the right-hand side being the negentropy gap between the moment-matched Gaussian and the two-component conditional mixture. Since Assumption~(A2) gives $\Sigma \succ 0$ and $\Delta \neq 0$, so that $\rho^2 = \Delta^\top\Sigma\Delta > 0$, the logarithmic term diverges as $\rho \to \infty$ while $h(\pi) \le \log 2$ stays bounded. Hence
\[
\liminf_{\rho \to \infty} D_{\rm Gauss}^*(\theta) \;\ge\; \liminf_{\rho \to \infty}\Bigl[ \frac{1}{2}\log\bigl(1 + \pi(1-\pi)\rho^2/\sigma^2\bigr) - h(\pi)\Bigr] = \infty,
\]
which is the correct direction for the infimum and avoids the spurious appeal to a single candidate.

For recoverability, $R_{\rm Gauss}^* = \E_X[I(Z;Y|X)] \le H(Z|X) = h(\pi) < \infty$ by the standard mutual information bound.

\emph{Binomial case ($D_m^*$ and $R_m^*$ both bounded).} Since $Y \in \{0, 1, \ldots, m\}$, both $P_\theta(Y|X)$ and any candidate $\Bin(m, q_\gamma(X))$ are distributions on a finite support. Under Assumption~(A1), $p_k(X), q_\gamma(X) \in (0,1)$ a.s., and $\Bin(y;m,p)$ is bounded below by $\min(p,1-p)^m > 0$ uniformly over $X$ in any compact set. Therefore $\KL(P_\theta(Y|X) \| \Bin(m, q_\gamma(X)))$ is bounded uniformly in $\rho$ for any fixed $\gamma$, and a fortiori
\[
D_m^*(\theta) = \inf_\gamma \E_X[\KL(\cdot)] \le \E_X[\KL(P_\theta(Y|X) \| \Bin(m, \pi))] \le C(m,\pi) < \infty
\]
uniformly in $\rho$. (The constant can be taken as $(m-1)h(\pi) + O(1)$ by considering the limiting point-mass mixture $\pi\delta_m + (1-\pi)\delta_0$.)

For recoverability, $R_m^* = \E_X[I(Z;Y|X)] \le h(\pi) < \infty$ as in the Gaussian case.
\end{proof}


\section{Proofs for Section~\ref{sec:local}}
\label{sec:SM-local}

The proof of the quartic law proceeds in several stages. We first establish a factorization of the binomial KL divergence, then develop pointwise and uniform Taylor expansions of the component KL, and finally assemble the quartic law first for the candidate divergence $D_{m,\mathrm{cand}}$ and then for the true divergence $D_m^*$.

\subsection{Binomial KL factorization}

\begin{lemma}[Binomial KL scaling]
\label{lem:SM-binomial-KL-scaling}
For any $m \in \{1, 2, \ldots\}$ and $p, q \in (0,1)$,
\[
\KL\bigl(\Bin(m, p) \| \Bin(m, q)\bigr) = m \cdot \KL\bigl(\Ber(p) \| \Ber(q)\bigr).
\]
\end{lemma}

\begin{proof}
The binomial probability mass function satisfies $\binom{m}{y} p^y (1-p)^{m-y} = \binom{m}{y} \exp\bigl(y \log\frac{p}{1-p} + m\log(1-p)\bigr)$. The KL divergence is
\begin{align*}
\KL(\Bin(m,p) \| \Bin(m,q)) &= \sum_{y=0}^{m} \binom{m}{y} p^y(1-p)^{m-y} \left[ y \log\frac{p}{q} + (m-y)\log\frac{1-p}{1-q}
\right] \\
&= mp \log\frac{p}{q} + m(1-p)\log\frac{1-p}{1-q} \\
&= m \cdot \KL(\Ber(p) \| \Ber(q)),
\end{align*}
where we used $\E[Y] = mp$ under $\Bin(m,p)$.
\end{proof}

\subsection{Expansion of the mixture probability}

\begin{lemma}[Pointwise second-order expansion]
\label{lem:SM-pointwise-second-order}
Under Assumptions~(A1)-(A3), for every $x \in \R^p$,
\[
\bar p_\theta(x) - q_0(x) = \frac{\pi(1-\pi)}{2}\psi''(x^\top \beta_0)(x^\top \Delta)^2 + R_3(x, \Delta),
\]
where $|R_3(x, \Delta)| \leq C|x^\top \Delta|^3$ for some constant $C > 0$ uniform over $(\pi, \beta_0) \in [\varepsilon, 1-\varepsilon] \times \cB$.
\end{lemma}

\begin{proof}
Let $t_0 = x^\top \beta_0$, $\delta_1 = (1-\pi)x^\top \Delta$, $\delta_2 = -\pi x^\top \Delta$. Then
\[
\bar p_\theta(x) = \pi\psi(t_0 + \delta_1) + (1-\pi)\psi(t_0 + \delta_2).
\]
By Taylor's theorem with $\psi''' $ bounded,
\[
\psi(t_0 + \delta) = \psi(t_0) + \psi'(t_0)\delta + \frac{1}{2}\psi''(t_0)\delta^2 + O(|\delta|^3).
\]
Substituting $\delta_1, \delta_2$, the linear term cancels:
\[
\pi\delta_1 + (1-\pi)\delta_2 = \pi(1-\pi)x^\top\Delta - (1-\pi)\pi x^\top\Delta = 0.
\]
The quadratic term gives
\[
\pi\delta_1^2 + (1-\pi)\delta_2^2 = \pi(1-\pi)(x^\top\Delta)^2.
\]
The remainder satisfies $|R_3(x,\Delta)| \leq C|x^\top\Delta|^3$.
\end{proof}

\begin{lemma}[Scaled stochastic expansion]
\label{lem:SM-scaled-expansion}
Under Assumptions~(A1)-(A3),
\[
\bar p_\theta(X) - q_0(X) = \rho(\Delta)^2 g_u(X) + r_\Delta(X),
\]
where
\[
\E\left[\frac{m r_\Delta(X)^2}{w_0(X)}\right] = o\bigl(\rho(\Delta)^4\bigr) \qquad \text{as } \rho(\Delta) \to 0.
\]
\end{lemma}

\begin{proof}
By Lemma~\ref{lem:SM-pointwise-second-order} and the substitution $X^\top\Delta = \rho(\Delta)Z_u$, the leading term is $\rho(\Delta)^2 g_u(X)$. The remainder $r_\Delta(X) = R_3(X,\Delta)$ satisfies $|r_\Delta(X)| \leq C|X^\top\Delta|^3 \leq C\|X\|^3\|\Delta\|^3$. Since $\|\Delta\|^3 = O(\rho(\Delta)^3)$ by $\Sigma \succ 0$,
\[
\E\left[\frac{m r_\Delta(X)^2}{w_0(X)}\right] \leq C\rho(\Delta)^6 \E\left[\frac{m\|X\|^6}{w_0(X)}\right] = o(\rho(\Delta)^4)
\]
by Assumption~\ref{ass:A2}.
\end{proof}

\begin{lemma}[Weighted square integrability of the leading term]
\label{lem:SM-leading-L2}
Under Assumptions~(A1)-(A3),
\[
\E\left[\frac{m g_u(X)^2}{w_0(X)}\right] < \infty.
\]
\end{lemma}

\begin{proof}
Since $|Z_u| \leq C\|X\|$ and $\pi \in [\varepsilon,1-\varepsilon]$,
\[
\frac{m g_u(X)^2}{w_0(X)} \leq C \frac{m\|X\|^4}{w_0(X)},
\]
which is integrable by Assumption~\ref{ass:A2}.
\end{proof}

\subsection{Local quadratic expansion of KL}

\begin{lemma}[Local quadratic expansion of Bernoulli KL]
\label{lem:SM-KL-local}
Let $q \in (0,1)$ and let $\delta, \eta$ be such that $q + \delta \in (0,1)$ and $q + \eta \in (0,1)$. Then
\[
\KL\bigl(\Ber(q+\delta) \| \Ber(q+\eta)\bigr) = \frac{(\delta - \eta)^2}{2q(1-q)} + R(q, \delta, \eta),
\]
where $|R(q, \delta, \eta)| \leq C_q(|\delta| + |\eta|)|\delta - \eta|^2$.
\end{lemma}

\begin{proof}
Define $\phi(a,b) = a\log\frac{a}{b} + (1-a)\log\frac{1-a}{1-b}$. Then $\KL(\Ber(a)\|\Ber(b)) = \phi(a,b)$. Expand $\phi$ around $(q,q)$: since $\phi(q,q) = 0$ and the gradient vanishes,
\[
\partial_{aa}\phi(q,q) = \frac{1}{q(1-q)}, \quad \partial_{ab}\phi(q,q) = -\frac{1}{q(1-q)}, \quad \partial_{bb}\phi(q,q) = \frac{1}{q(1-q)}.
\]
The leading quadratic form is $\frac{(\delta-\eta)^2}{2q(1-q)}$. The remainder bound follows from Taylor's theorem and local boundedness of third derivatives.
\end{proof}

\begin{lemma}[Random-version KL expansion]
\label{lem:SM-KL-random}
Let $\delta_\Delta(X)$ and $\eta_\Delta(X)$ be random perturbations with
\[
\E\left[\frac{m\delta_\Delta(X)^2}{w_0(X)}\right] = O(r_\Delta^2), \qquad \E\left[\frac{m\eta_\Delta(X)^2}{w_0(X)}\right] = O(r_\Delta^2),
\]
for some $r_\Delta \to 0$, and
\[
\E\left[\frac{m(|\delta_\Delta(X)| + |\eta_\Delta(X)|)|\delta_\Delta(X) - \eta_\Delta(X)|^2}{w_0(X)}\right] = o(r_\Delta^2).
\]
Then
\[
\E\left[ m\cdot\KL\bigl(\Ber(q_0(X) + \delta_\Delta(X)) \| \Ber(q_0(X) + \eta_\Delta(X))\bigr) \right] = \frac{1}{2}\E\left[\frac{m(\delta_\Delta(X) - \eta_\Delta(X))^2}{w_0(X)}\right] + o(r_\Delta^2).
\]
\end{lemma}

\begin{proof}
Apply Lemma~\ref{lem:SM-KL-local} pointwise with $q = q_0(X)$, $\delta = \delta_\Delta(X)$, $\eta = \eta_\Delta(X)$, multiply by $m$, and integrate. The remainder is controlled by the assumed bound.
\end{proof}

\subsection{Local expansion of the one-component family}

\begin{lemma}[Local expansion of the one-component family]
\label{lem:SM-onecomp-local}
Under Assumptions~(A1)-(A3), for $\|a\| \to 0$,
\[
q_{\beta_0+a}(X) = q_0(X) + w_0(X) X^\top a + \widetilde{R}(X, a),
\]
where $|\widetilde{R}(X, a)| \leq C|X^\top a|^2$ and
\[
\E\left[\frac{m\widetilde{R}(X, a)^2}{w_0(X)}\right] = O(\|a\|^4).
\]
\end{lemma}

\begin{proof}
Taylor expand $\psi(X^\top(\beta_0 + a))$ around $a = 0$. The derivative at $a = 0$ is $w_0(X)X$, so
\[
q_{\beta_0+a}(X) = q_0(X) + w_0(X) X^\top a + \widetilde{R}(X,a)
\]
with $\widetilde{R}(X,a) = O(|X^\top a|^2)$. Then
\[
\frac{m\widetilde{R}(X,a)^2}{w_0(X)} \leq C\|a\|^4 \frac{m\|X\|^4}{w_0(X)},
\]
which has finite expectation by Assumption~\ref{ass:A2}.
\end{proof}

\begin{lemma}[Positive definiteness of the weighted information matrix]
\label{lem:SM-weighted-info-pd}
Under Assumptions~(A1)-(A3),
\[
M_{0,m} := \E[m w_0(X) XX^\top] \succ 0.
\]
\end{lemma}

\begin{proof}
For $v \in \R^p \setminus \{0\}$,
\[
v^\top M_{0,m} v = \E[m w_0(X) (v^\top X)^2] > 0,
\]
since $m \geq 1$, $w_0(X) > 0$ a.s., and $\E[(v^\top X)^2] = v^\top \Sigma v > 0$.
\end{proof}

\subsection{Profile problem and quartic laws}

\begin{proposition}[Reduction of the profile problem to the $\rho(\Delta)^2$ scale]
\label{prop:SM-profile-reduction}
Under Assumptions~(A1)-(A3), there exists $M < \infty$ such that
\[
D_m^*(\theta) = \inf_{\|b\| \leq M} L_\Delta(\rho(\Delta)^2 b) + o(\rho(\Delta)^4).
\]
\end{proposition}

\begin{proof}
Recall $L_{\Delta}(a):=\E\left[\KL\Bigl(\Bin\bigl(m,\bar p_{\theta}(X)\bigr)\Big\|\Bin\bigl(m,q_{\beta_0+a}(X)\bigr)\Bigr)\right]$, so that $D_m^*(\theta)=\inf_{a\in\R^p}L_{\Delta}(a)$.

We first show that the global infimum is attained in a shrinking neighborhood of the origin. Let $\varepsilon_0>0$ be small enough that $\beta_0+a\in\cB$ whenever $\|a\|\le \varepsilon_0$. Because $\bar p_\theta(\cdot)\to q_0(\cdot)$ pointwise as $\rho(\Delta)\to0$, and because the binomial KL divergence is continuous in both arguments on compact subsets of $(0,1)$, we have $L_{\Delta}(a)\to L_0(a):=\E\left[\KL\Bigl(\Bin\bigl(m,q_0(X)\bigr)\Big\|\Bin\bigl(m,q_{\beta_0+a}(X)\bigr)\Bigr)\right]$ for each fixed $a$. Moreover, $L_0(a)=0$ if and only if $q_{\beta_0+a}(X)=q_0(X)$ almost surely, which by strict monotonicity of the logistic link and $\Sigma\succ0$ implies $a=0$. Hence the continuous map $a\mapsto L_0(a)$ is strictly positive on the compact shell $K_{\varepsilon_0}:=\{a:\|a\|=\varepsilon_0\}$, so $\inf_{\|a\|\ge \varepsilon_0}L_0(a)\ge c_0$ for some $c_0>0$. By compact convergence, for all sufficiently small $\rho(\Delta)$, $\inf_{\|a\|\ge \varepsilon_0}L_{\Delta}(a)\ge c_0/2$. On the other hand, by the candidate quartic law (Theorem~\ref{thm:SM-candidate-quartic}), $L_{\Delta}(0)=D_{m,\mathrm{cand}}(\theta)=O\bigl(\rho(\Delta)^4\bigr)$. Since $\rho(\Delta)^4\to0$, every near-minimizer of $L_{\Delta}$ must lie in the local region $\|a\|<\varepsilon_0$ once $\rho(\Delta)$ is small enough.

We next analyze $L_\Delta(a)$ inside this local region. By Lemma~\ref{lem:SM-scaled-expansion},
\[
\bar p_\theta(X)-q_0(X)=\rho(\Delta)^2 g_u(X)+r_\Delta(X), \qquad \E\left[\frac{mr_\Delta(X)^2}{w_0(X)}\right] = o\bigl(\rho(\Delta)^4\bigr),
\]
and by Lemma~\ref{lem:SM-onecomp-local},
\[
q_{\beta_0+a}(X)-q_0(X)=w_0(X)X^\top a+\widetilde{R}(X,a), \qquad \E\left[\frac{m\widetilde{R}(X,a)^2}{w_0(X)}\right] = O(\|a\|^4).
\]
Applying Lemma~\ref{lem:SM-KL-random} gives, uniformly over $\|a\|\le \varepsilon_0$,
\[
L_\Delta(a) = \frac12 \E\left[ \frac{m\bigl(\rho(\Delta)^2 g_u(X)-w_0(X)X^\top a\bigr)^2}{w_0(X)} \right] + o\bigl(\rho(\Delta)^4+\|a\|^2\bigr).
\]
Expanding the square,
\[
L_\Delta(a) = \frac12 \rho(\Delta)^4 \E\left[\frac{mg_u(X)^2}{w_0(X)}\right] - \rho(\Delta)^2 \E[mg_u(X)X]^\top a + \frac12 a^\top M_{0,m} a + o\bigl(\rho(\Delta)^4+\|a\|^2\bigr),
\]
where $M_{0,m}=\E[mw_0(X)XX^\top]$. By Lemma~\ref{lem:SM-weighted-info-pd}, $M_{0,m}\succ0$. Let $\lambda_*>0$ denote its smallest eigenvalue and $B_*:=\bigl\|\E[mg_u(X)X]\bigr\|$. For $\|a\|\le\varepsilon_0$,
\[
L_\Delta(a) \ge \frac{\lambda_*}{2}\|a\|^2 - \rho(\Delta)^2 B_* \|a\| - C\rho(\Delta)^4 + o\bigl(\rho(\Delta)^4+\|a\|^2\bigr).
\]
Using $\rho(\Delta)^2 B_* \|a\|\le \frac{\lambda_*}{4}\|a\|^2 + \frac{B_*^2}{\lambda_*}\rho(\Delta)^4$, we obtain $L_\Delta(a)\ge \frac{\lambda_*}{4}\|a\|^2-C'\rho(\Delta)^4+o\bigl(\rho(\Delta)^4+\|a\|^2\bigr)$. Therefore, if $\|a\|\ge M\rho(\Delta)^2$ and $\rho(\Delta)$ is small, $L_\Delta(a)\ge (\lambda_*M^2/8-C'')\rho(\Delta)^4$. Choosing $M$ so that $\lambda_*M^2/8>C''+1$ gives $\inf_{\|a\|\ge M\rho(\Delta)^2}L_\Delta(a)\ge \rho(\Delta)^4$. Since $L_\Delta(0)=O(\rho(\Delta)^4)$, the infimum cannot be attained outside $\|a\|\le M\rho(\Delta)^2$.

Setting $a=\rho(\Delta)^2 b$ and noting $\|a\|\le M\rho(\Delta)^2\Leftrightarrow\|b\|\le M$ gives $D_m^*(\theta)=\inf_{\|b\|\le M}L_\Delta\bigl(\rho(\Delta)^2 b\bigr)+o\bigl(\rho(\Delta)^4\bigr)$.
\end{proof}

\begin{lemma}[Quadratic approximation of the profile criterion]
\label{lem:SM-profile-quadratic}
Fix $M < \infty$. For $b \in \R^p$ with $\|b\| \leq M$, set $a = \rho(\Delta)^2 b$. Under Assumptions~(A1)-(A3),
\[
L_\Delta(\rho(\Delta)^2 b) = \frac{\rho(\Delta)^4}{2} \E\left[ \frac{m\bigl(g_u(X) - w_0(X) X^\top b\bigr)^2}{w_0(X)} \right] + o(\rho(\Delta)^4),
\]
uniformly over $\|b\| \leq M$.
\end{lemma}

\begin{proof}
By Lemma~\ref{lem:SM-scaled-expansion},
\[
\bar p_\theta(X) = q_0(X) + \rho(\Delta)^2 g_u(X) + r_\Delta(X).
\]
By Lemma~\ref{lem:SM-onecomp-local},
\[
q_{\beta_0 + \rho(\Delta)^2 b}(X) = q_0(X) + \rho(\Delta)^2 w_0(X) X^\top b + \widetilde{r}_\Delta(X, b),
\]
where $\E[m\widetilde{r}_\Delta(X,b)^2/w_0(X)] = O(\rho(\Delta)^8) = o(\rho(\Delta)^4)$ uniformly over $\|b\| \leq M$. Therefore
\[
\bar p_\theta(X) - q_{\beta_0 + \rho(\Delta)^2 b}(X) = \rho(\Delta)^2\bigl(g_u(X) - w_0(X) X^\top b\bigr) + \bar{r}_\Delta(X, b),
\]
with $\E[m\bar{r}_\Delta(X,b)^2/w_0(X)] = o(\rho(\Delta)^4)$ uniformly. Apply Lemma~\ref{lem:SM-KL-random} with $\delta_\Delta = \rho(\Delta)^2 g_u + r_\Delta$ and $\eta_\Delta = \rho(\Delta)^2 w_0 X^\top b + \widetilde{r}_\Delta$.
\end{proof}

\begin{lemma}[Projection formula]
\label{lem:SM-projection-formula}
Define
\[
\Psi_m(b) := \frac{1}{2} \E\left[ \frac{m\bigl(g_u(X) - w_0(X) X^\top b\bigr)^2}{w_0(X)} \right].
\]
Then $\Psi_m$ is strictly convex with unique minimizer
\[
b_m^* = \bigl(\E[m w_0(X) XX^\top]\bigr)^{-1} \E[m g_u(X) X].
\]
\end{lemma}

\begin{proof}
Expanding the square,
\[
\Psi_m(b) = \frac{1}{2}\E\left[\frac{m g_u(X)^2}{w_0(X)}\right] - \E[m g_u(X) X]^\top b + \frac{1}{2} b^\top \E[m w_0(X) XX^\top] b.
\]
By Lemma~\ref{lem:SM-weighted-info-pd}, the Hessian $M_{0,m} = \E[m w_0(X) XX^\top] \succ 0$, so $\Psi_m$ is strictly convex. Setting the gradient to zero gives the formula for $b_m^*$.
\end{proof}

\subsection{Candidate quartic law}

\begin{theorem}[Candidate quartic law]
\label{thm:SM-candidate-quartic}
Under Assumptions~(A1)-(A3),
\[
D_{m,\mathrm{cand}}(\theta) = C_{m,\mathrm{cand}}(\pi, \beta_0, u)\rho(\Delta)^4 + o(\rho(\Delta)^4),
\]
where
\[
C_{m,\mathrm{cand}}(\pi, \beta_0, u) = \frac{1}{2} \E\left[\frac{m g_u(X)^2}{w_0(X)}\right] = \frac{(\pi(1-\pi))^2}{8} \E\left[\frac{m\psi''(X^\top \beta_0)^2 Z_u^4}{w_0(X)}\right].
\]
\end{theorem}

\begin{proof}
Take $b = 0$ in Lemma~\ref{lem:SM-profile-quadratic}:
\[
L_\Delta(0) = \frac{\rho(\Delta)^4}{2}\E\left[\frac{m g_u(X)^2}{w_0(X)}\right] + o(\rho(\Delta)^4).
\]
Since $L_\Delta(0) = D_{m,\mathrm{cand}}(\theta)$ by Lemma~\ref{lem:SM-binomial-KL-scaling}, the result follows.
\end{proof}

\subsection{Overdispersion contribution to $D_m^*$}

The following lemma supplies the $\bar\lambda$ contribution to the quartic law of $D_m^*$, which together with the Bernoulli-level part assembled in the proof of Theorem~\ref{thm:true-quartic} below yields the full decomposition $J_m = m\kappa + m(m-1)\bar\lambda/4$.

\begin{lemma}[Quartic law of the mixture-vs-matched-binomial KL]
\label{lem:SM-T1-quartic}
Under Assumptions~(A1)--(A3), as $\rho(\Delta) \to 0$,
\[
T_1(\theta) := \E_X\left[\KL\bigl(P_\theta(Y|X) \big\| \Bin(m, \bar p_\theta(X))\bigr)\right] = \frac{m(m-1)}{4}\bar\lambda\rho(\Delta)^4 + o(\rho(\Delta)^4),
\]
where $\bar\lambda = [\pi(1-\pi)]^2 \E[\psi'(X^\top\beta_0)^2 (\rho^{-1}\Delta^\top X)^4]$.
\end{lemma}

\begin{proof}
Fix $X = x$ and write $\bar p = \bar p_\theta(x)$, $\delta = \delta(x) = p_1(x) - p_2(x)$, so that $p_1 = \bar p + (1-\pi)\delta$ and $p_2 = \bar p - \pi\delta$. Let $s_y(p) = \partial_p \log\Bin(y;m,p) = (y - mp)/[p(1-p)]$ and $t_y(p) = \partial_p s_y(p)$.

By Taylor expansion of $\Bin(y;m,p)$ around $\bar p$,
\[
\Bin(y;m,p) = \Bin(y;m,\bar p)\bigl[1 + (p-\bar p)s_y(\bar p) + \frac{(p-\bar p)^2}{2}\bigl(s_y(\bar p)^2 + t_y(\bar p)\bigr) + O((p-\bar p)^3)\bigr]
\]
uniformly in $y \in \{0,\ldots,m\}$, since $\bar p$ is bounded away from $\{0,1\}$ on $\cB$. Applying this to $p \in \{p_1, p_2\}$ and forming the mixture, the first-order term cancels by $\sum_{k=1}^K \pi_k(p_k - \bar p) = 0$, while $\sum_{k=1}^K \pi_k(p_k - \bar p)^2 = \pi(1-\pi)\delta^2$. Hence
\[
\frac{P_\theta(Y=y|X=x)}{\Bin(y;m,\bar p)} = 1 + \phi_y + O(\delta^3), \qquad \phi_y := \frac{1}{2}\pi(1-\pi)\delta^2\bigl(s_y(\bar p)^2 + t_y(\bar p)\bigr).
\]

Using $(1+u)\log(1+u) = u + u^2/2 + O(u^3)$ with $u = \phi_y + O(\delta^3)$,
\begin{align*}
\KL\bigl(P_\theta(Y|X=x) \| \Bin(m,\bar p)\bigr) 
&= \E_{\Bin(m,\bar p)}\left[(1 + \phi_Y + O(\delta^3))\log(1 + \phi_Y + O(\delta^3))\right] \\
&= \E_{\Bin(m,\bar p)}[\phi_Y] + \frac{1}{2}\E_{\Bin(m,\bar p)}[\phi_Y^2] + O(\delta^5).
\end{align*}
By the second Bartlett identity, $\E_{\Bin(m,\bar p)}[s_Y(\bar p)^2 + t_Y(\bar p)] = 0$, hence $\E_{\Bin(m,\bar p)}[\phi_Y] = 0$. The cubic correction from the $O(\delta^3)$ term inside the log integrates to zero by the third Bartlett identity $\E[s^3 + 3st + \partial_p t] = 0$. Therefore
\[
\KL\bigl(P_\theta(Y|X=x) \| \Bin(m,\bar p)\bigr) = \frac{1}{2}\E_{\Bin(m,\bar p)}[\phi_Y^2] + O(\delta^5)
= \frac{1}{8}[\pi(1-\pi)]^2 \delta^4 V(\bar p) + O(\delta^5),
\]
where $V(\bar p) := \mathrm{Var}_{\Bin(m,\bar p)}(s_Y(\bar p)^2 + t_Y(\bar p))$.

To compute $V(\bar p)$, let $U = (Y - m\bar p)/\sqrt{m\bar p(1-\bar p)}$. Then $s_Y = U\sqrt{a}$ and $t_Y = -a - U\sqrt{a}b$, where $a = m/[\bar p(1-\bar p)]$ and $b = (1-2\bar p)/[\bar p(1-\bar p)]$, so
\[
s_Y^2 + t_Y = a(U^2 - 1) - U\sqrt{a}b.
\]
Using $\mathrm{Var}(U) = 1$, $\mathrm{Var}(U^2) = 2 + (1 - 6\bar p(1-\bar p))/[m\bar p(1-\bar p)]$, and $\mathrm{Cov}(U^2, U) = \E[U^3] = (1 - 2\bar p)/\sqrt{m\bar p(1-\bar p)}$,
\begin{align*}
V(\bar p) 
&= a^2 \mathrm{Var}(U^2) + a b^2 - 2 a^{3/2} b \mathrm{Cov}(U^2, U) \\
&= \frac{2m^2}{[\bar p(1-\bar p)]^2} + \frac{m \cdot \{1 - 6\bar p(1-\bar p) - (1-2\bar p)^2\}}{[\bar p(1-\bar p)]^3}.
\end{align*}
Since $1 - (1-2\bar p)^2 = 4\bar p(1-\bar p)$, the bracket equals $-2\bar p(1-\bar p)$, hence
\[
V(\bar p) = \frac{2m(m-1)}{[\bar p(1-\bar p)]^2}.
\]

Substituting back,
\[
\KL\bigl(P_\theta(Y|X=x) \| \Bin(m,\bar p)\bigr) = \frac{m(m-1)}{4}[\pi(1-\pi)]^2 \frac{\delta(x)^4}{[\bar p(x)(1-\bar p(x))]^2} + O(\delta(x)^5).
\]
By Lemma~\ref{lem:SM-prob-difference}, $\delta(x) = \psi'(x^\top\beta_0)x^\top\Delta + O((x^\top\Delta)^2)$, so $\delta(x)^4 = \psi'(x^\top\beta_0)^4 (x^\top\Delta)^4 + O(\rho^5\|x\|^5)$. By Lemma~\ref{lem:SM-pointwise-second-order}, $\bar p(x) = q_0(x) + O(\rho^2)$, hence $\bar p(x)(1-\bar p(x)) = w_0(x) + O(\rho^2)$ and $[\bar p(x)(1-\bar p(x))]^{-2} = w_0(x)^{-2} + O(\rho^2)$. Using $w_0 = \psi'$ and $x^\top\Delta = \rho Z_u$,
\[
\frac{\delta(x)^4}{[\bar p(x)(1-\bar p(x))]^2} = \psi'(x^\top\beta_0)^2\rho^4 Z_u^4 + O(\rho^5\|x\|^5).
\]
Taking expectations (justified by Assumptions~(A2), (R2)--(R3)),
\[
T_1(\theta) = \frac{m(m-1)}{4}[\pi(1-\pi)]^2 \rho^4 \E[\psi'(X^\top\beta_0)^2 Z_u^4] + o(\rho^4)
= \frac{m(m-1)}{4}\bar\lambda\rho^4 + o(\rho^4). \qedhere
\]
\end{proof}

\subsection{Proof of Theorem~\ref{thm:true-quartic}}
\begin{proof}[Proof of Theorem~\ref{thm:true-quartic}]
By Lemma~\ref{lem:kl-decomposition}, for any $\gamma \in \R^p$,
\[
\E_X[\KL(P_\theta(Y|X) \| \Bin(m, q_\gamma(X)))] = T_1(\theta) + m \E_X[\KL(\Ber(\bar p_\theta(X)) \| \Ber(q_\gamma(X)))],
\]
where
\[
T_1(\theta) := \E_X[\KL(P_\theta(Y|X) \| \Bin(m, \bar p_\theta(X)))]
\]
does not depend on $\gamma$. Hence
\[
D_m^*(\theta) = T_1(\theta) + m \inf_{\gamma \in \R^p}\E_X[\KL(\Ber(\bar p_\theta(X)) \| \Ber(q_\gamma(X)))].
\]

Define
\[
L_{\Delta}^{\rm Ber}(a) := \E_X\left[ \KL\Bigl( \Ber\bigl(\bar p_{\theta}(X)\bigr) \Big\| \Ber\bigl(q_{\beta_0+a}(X)\bigr) \Bigr) \right],
\]
so that the second term equals $m \inf_a L_\Delta^{\rm Ber}(a)$. We show $\inf_a L_\Delta^{\rm Ber}(a) = \kappa\rho(\Delta)^4 + o(\rho(\Delta)^4)$.

The profile reduction: there exists a fixed $M<\infty$ such that
\[
\inf_a L_\Delta^{\rm Ber}(a) = \inf_{\|b\|\le M} L_{\Delta}^{\rm Ber}\bigl(\rho(\Delta)^2 b\bigr) +o\bigl(\rho(\Delta)^4\bigr).
\]
Any near-minimizer must satisfy $a=O(\rho(\Delta)^2)$; if $\|a\|$ is larger than a sufficiently large multiple of $\rho(\Delta)^2$, then the quadratic term from the local expansion of the one-component family dominates.

For $\|b\|\le M$, the stochastic expansion of the mixture mean (Lemma~\ref{lem:SM-scaled-expansion}) and the local expansion of the one-component family (Lemma~\ref{lem:SM-onecomp-local}) imply
\[
\bar p_{\theta}(X) - q_{\beta_0+\rho(\Delta)^2 b}(X) = \rho(\Delta)^2\bigl(g_u(X)-w_0(X)X^\top b\bigr) + \bar r_\Delta(X, b),
\]
where $\bar r_\Delta(X, b) := r_\Delta(X) - \widetilde{R}\bigl(X, \rho(\Delta)^2 b\bigr)$ aggregates the mixture-side remainder $r_\Delta$ from Lemma~\ref{lem:SM-scaled-expansion} with the one-component-side remainder $\widetilde{R}$ from Lemma~\ref{lem:SM-onecomp-local}. By the bounds in those lemmas, $\E[\bar r_\Delta(X,b)^2/w_0(X)] = o(\rho(\Delta)^4)$ uniformly over $\|b\|\le M$. Applying Lemma~\ref{lem:SM-KL-random} (with $m \equiv 1$, since $L_\Delta^{\rm Ber}$ is the Bernoulli KL),
\[
L_{\Delta}^{\rm Ber}\bigl(\rho(\Delta)^2 b\bigr) = \rho(\Delta)^4 \Psi(b) + o\bigl(\rho(\Delta)^4\bigr),
\]
uniformly in $\|b\|\le M$, where
\[
\Psi(b) := \frac12 \E\left[ \frac{\bigl(g_u(X)-w_0(X)X^\top b\bigr)^2}{w_0(X)} \right].
\]
By Lemma~\ref{lem:SM-projection-formula} (with $m \equiv 1$), $\Psi$ is strictly convex with unique minimizer in a bounded set; choosing $M$ large enough,
\[
\inf_{b}\Psi(b) = \kappa(\pi, \beta_0, u),
\]
matching the definition of $\kappa$ in the statement. Hence $\inf_a L_\Delta^{\rm Ber}(a) = \kappa\rho(\Delta)^4 + o(\rho(\Delta)^4)$, and the second term contributes $m\kappa\rho(\Delta)^4 + o(\rho(\Delta)^4)$.
By Lemma~\ref{lem:SM-T1-quartic},
\[
T_1(\theta) = \frac{m(m-1)}{4}\bar\lambda\rho(\Delta)^4 + o(\rho(\Delta)^4),
\]
with $\bar\lambda$ as in the statement.
Then, we have 
$$
D_m^*(\theta) = T_1(\theta) + m\kappa\rho(\Delta)^4 + o(\rho(\Delta)^4) = J_m\rho(\Delta)^4 + o(\rho(\Delta)^4),
$$
where $J_m = m\kappa + m(m-1)\bar\lambda/4$.
\end{proof}

\subsection{Auxiliary lemmas for the proof of Theorem 2}

The proof proceeds by expanding the conditional mutual information to second order.

\begin{lemma}[Local expansion of the component probability difference]
\label{lem:SM-prob-difference}
Under Assumption~(A1), for each $x \in \R^p$,
\[
\delta(x) := p_1(x) - p_2(x) = \psi'(x^\top \beta_0) x^\top \Delta + R_2(x, \Delta),
\]
where $|R_2(x, \Delta)| \leq C(x^\top\Delta)^2$ for a constant $C$ uniform over $\cB$.
\end{lemma}

\begin{proof}
By the mean value theorem,
\begin{align*}
p_1(x) - p_2(x) &= \psi(x^\top\beta_0 + (1-\pi)x^\top\Delta) - \psi(x^\top\beta_0 - \pi x^\top\Delta).
\end{align*}
Taylor expanding each around $x^\top\beta_0$:
\begin{align*}
\psi(x^\top\beta_0 + (1-\pi)x^\top\Delta)
&= \psi(x^\top\beta_0) + \psi'(x^\top\beta_0)(1-\pi)x^\top\Delta + O((x^\top\Delta)^2), \\
\psi(x^\top\beta_0 - \pi x^\top\Delta) &= \psi(x^\top\beta_0) - \psi'(x^\top\beta_0)\pi x^\top\Delta + O((x^\top\Delta)^2).
\end{align*}
Subtracting:
\[
\delta(x) = \psi'(x^\top\beta_0)\bigl((1-\pi) + \pi\bigr)x^\top\Delta + O((x^\top\Delta)^2) = \psi'(x^\top\beta_0) x^\top\Delta + R_2(x, \Delta). \qedhere
\]
\end{proof}

\begin{lemma}[Small-signal mutual information expansion]
\label{lem:SM-small-signal-MI}
Consider a binary latent variable $Z \in \{1,2\}$ with $\Pr(Z=1) = \pi$, and let $Y | Z=k \sim \Bin(m, p_k)$ where $p_1 = \bar{p} + (1-\pi)\delta$ and $p_2 = \bar{p} - \pi\delta$ for some $\bar{p} \in (0,1)$ and $\delta$ small. Then, as $\delta \to 0$,
\[
I(Z; Y) = \frac{\pi(1-\pi)}{2} \cdot \frac{m \delta^2}{\bar{p}(1-\bar{p})} + O(m^2\delta^4 + m\delta^3).
\]
\end{lemma}

\begin{proof}
Fix $m\ge1$, $\pi\in[\varepsilon,1-\varepsilon]$, and $\bar p$ in a compact subset of $(0,1)$. Let $q_y(p):=\binom{m}{y} p^y(1-p)^{m-y}$ for $y=0,\dots,m$, and define $p_1=\bar p+(1-\pi)\delta$, $p_2=\bar p-\pi\delta$. Write $q_y:=q_y(\bar p)$ and
\[
s_y:=\partial_p \log q_y(p)\big|_{p=\bar p} = \frac{y-m\bar p}{\bar p(1-\bar p)}.
\]
Because the binomial support is finite and $\bar p$ stays away from $0$ and $1$, derivatives of $q_y(p)$ up to third order are uniformly bounded by constants polynomial in $m$. Taylor expansion around $\bar p$ gives, uniformly in $y$,
\[
q_y(p_k) = q_y\Bigl\{ 1+a_k\delta s_y + \frac12 a_k^2\delta^2 \bigl(s_y^2+t_y\bigr) + r_{k,y} \Bigr\},
\]
where $a_1=1-\pi$, $a_2=-\pi$, $t_y:=\partial_p s_y|_{p=\bar p}$, and $\sum_{y=0}^m q_y |r_{k,y}|\le C\bigl(m|\delta|^3+m^2\delta^4\bigr)$.

Let $q_y^{\mathrm{mix}}:=\pi q_y(p_1)+(1-\pi)q_y(p_2)$. Since $\pi a_1+(1-\pi)a_2=0$, the first-order term cancels and
\[
q_y^{\mathrm{mix}} = q_y\Bigl\{ 1+ \frac12 \pi(1-\pi)\delta^2 (s_y^2+t_y) + \bar r_y \Bigr\},
\]
with $\sum_{y=0}^m q_y |\bar r_y|\le C\bigl(m|\delta|^3+m^2\delta^4\bigr)$.

The conditional mutual information is $I(Z;Y)=\pi \KL\bigl(q(\cdot;p_1)\|q^{\mathrm{mix}}\bigr)+(1-\pi)\KL\bigl(q(\cdot;p_2)\|q^{\mathrm{mix}}\bigr)$. We expand the first term; the second is analogous. Set $u_{1,y}:=a_1\delta s_y + O(\delta^2 b_y)$ and $v_y:=O(\delta^2 b_y)$ where $b_y:=1+s_y^2+|t_y|$. Then
\[
\KL\bigl(q(\cdot;p_1)\|q^{\mathrm{mix}}\bigr) = \sum_{y=0}^m q_y(1+u_{1,y}) \log\frac{1+u_{1,y}}{1+v_y}.
\]
Using the third-order Taylor expansion of $(1+u)\log\{(1+u)/(1+v)\}$ around $(u,v)=(0,0)$:
\[
(1+u)\log\frac{1+u}{1+v} = (u-v)+\frac12 (u-v)^2 + v(u-v) + O(|u|^3+|v|^3).
\]
Because $\sum_y q_y s_y=0$, the first-order term vanishes after summation. Since $v_y=O(\delta^2 b_y)$, all terms involving $v_y$ contribute only $O(m|\delta|^3+m^2\delta^4)$. Hence
\[
\KL\bigl(q(\cdot;p_1)\|q^{\mathrm{mix}}\bigr) = \frac12 (1-\pi)^2\delta^2 \sum_{y=0}^m q_y s_y^2 + O(m|\delta|^3+m^2\delta^4).
\]
Since $\sum_{y=0}^m q_y s_y^2 = \mathrm{Var}_{\bar p}(s_Y) + (\E_{\bar p}[s_Y])^2 = m/\{\bar p(1-\bar p)\}$ (the Fisher information of $\Bin(m,\bar p)$), the $k=1$ KL term is $\frac12(1-\pi)^2\delta^2 \cdot m/\{\bar p(1-\bar p)\}+O(m|\delta|^3+m^2\delta^4)$. The $k=2$ term gives $\frac12 \pi^2\delta^2\cdot m/\{\bar p(1-\bar p)\}+O(\cdot)$. Multiplying and adding:
\[
I(Z;Y) = \frac{\pi(1-\pi)}{2}\cdot\frac{m\delta^2}{\bar p(1-\bar p)} + O(m^2\delta^4+m\delta^3). \qedhere
\]
\end{proof}

\subsection{Proof of Theorem~\ref{thm:quadratic-recoverability}}
\begin{proof}
By Lemma~\ref{lem:SM-prob-difference}, conditional on $X = x$,
\[
\delta(x) = p_1(x) - p_2(x) = \psi'(x^\top\beta_0)x^\top\Delta + R_2(x, \Delta), \qquad |R_2(x, \Delta)| \le C(x^\top\Delta)^2.
\]
Substituting $x^\top\Delta = \rho(\Delta) Z_u$,
\[
\delta(x) = \rho(\Delta)\psi'(x^\top\beta_0) Z_u + O(\rho(\Delta)^2 Z_u^2).
\]

Applying Lemma~\ref{lem:SM-small-signal-MI} pointwise with $\bar p = q_0(x)$ (using $\bar p(x) = q_0(x) + O(\rho^2)$ from Lemma~\ref{lem:SM-pointwise-second-order}, which is absorbed in the $o(\rho^2)$ remainder),
\[
I_\theta(Z;Y | X=x) = \frac{\pi(1-\pi)}{2} \cdot \frac{m\psi'(x^\top\beta_0)^2 (x^\top\Delta)^2}{w_0(x)} + r(x, \Delta),
\]
with $|r(x,\Delta)| \le C[m^2 \delta(x)^4 + m\delta(x)^3]$. Since $\delta(x) = O(|x^\top\Delta|) = O(\rho \|x\|)$, Assumptions~(R2)--(R3) (uniform integrability of $m^2 \psi'^2 \|X\|^4 / w_0$ and dominated envelope) give
\[
\E[|r(X, \Delta)|] = O(\rho^4 \E[m^2 |Z_u|^4]) + O(\rho^3 \E[m |Z_u|^3]) = o(\rho^2).
\]
Using $\psi'^2(x^\top\beta_0)/w_0(x) = \psi'(x^\top\beta_0)$ and $(x^\top\Delta)^2 = \rho^2 Z_u^2$, the leading term integrates to
\[
\E_X\left[\frac{\pi(1-\pi)}{2} \cdot \frac{m\psi'(X^\top\beta_0)^2 (X^\top\Delta)^2}{w_0(X)}\right] = m\xi\rho^2,
\]
with $\xi = 2^{-1}\pi(1-\pi)\E[\psi'(X^\top\beta_0) Z_u^2]$ as in the statement. 
Combining the above results, we have
\[
R_m^*(\theta) = \E_X[I_\theta(Z;Y|X)] = m\xi\rho^2 + o(\rho^2). \qedhere
\]
\end{proof}

\subsection{Proof of Proposition~\ref{thm:oracle-impossibility}}
\begin{proof}
By the standard identity for conditional mutual information and the fact that $Z \perp X$ under the mixture model,
\[
R_m^*(\theta) = \E_X[I(Z; Y | X)] = \E_X\E_{Y|X}[\KL(P_\theta(Z | X, Y) \| P_\theta(Z))],
\]
where we used $P_\theta(Z|X) = P_\theta(Z) = \pi_z$. 
Pinsker's inequality gives, for each $(x, y)$,
\[
\|P_\theta(Z | X=x, Y=y) - P_\theta(Z)\|_{\mathrm{TV}} \le \sqrt{\frac{1}{2}\KL(P_\theta(Z | X=x, Y=y) \| P_\theta(Z))}.
\]
Taking expectation over $(X, Y)$ and applying Jensen's inequality (concavity of $\sqrt{\cdot}$),
\[
\E_{\theta_n}\left[\|P_{\theta_n}(Z | X, Y) - P_{\theta_n}(Z)\|_{\mathrm{TV}}\right] \le \sqrt{\frac{1}{2} R_m^*(\theta_n)} \to 0
\]
as $R_m^*(\theta_n) \to 0$. We pass from this total-variation collapse to the misclassification rate by comparing expected accuracies, not by claiming that the oracle and prior-only rules select the same label. They need not, since at $\pi_*=1/2$ the prior mode is arbitrary while the posterior mode remains informative on an event of vanishing probability. Write $a := P_{\theta_n}(Z=1 | X, Y)$ and $\pi_n := P_{\theta_n}(Z=1)$. On the two-point label space the oracle Bayes classifier $\hat Z(X,Y) = \argmax_z P_{\theta_n}(Z = z | X, Y)$ has conditional accuracy $\max(a, 1-a)$, the prior-only rule has accuracy $\max(\pi_n, 1-\pi_n)$, and $|a-\pi_n| = \|P_{\theta_n}(Z| X,Y) - P_{\theta_n}(Z)\|_{\mathrm{TV}}$. The map $t \mapsto \max(t, 1-t)$ is bounded by $1$ and $1$-Lipschitz, so
\begin{align*}
&\bigl| \E_{\theta_n}[\max(a, 1-a)] - \max(\pi_n, 1-\pi_n) \bigr|\\
&\le \E_{\theta_n}\left[ |a - \pi_n| \right]
= \E_{\theta_n}\left[ \|P_{\theta_n}(Z| X,Y) - P_{\theta_n}(Z)\|_{\mathrm{TV}} \right] \to 0,
\end{align*}
using only boundedness and the Lipschitz property, with no uniform integrability required. Since $\max(\pi_n, 1-\pi_n) \to \max(\pi_*, 1-\pi_*)$, the oracle accuracy converges to $\max(\pi_*, 1-\pi_*)$, i.e.\ the misclassification error converges to $\min(\pi_*, 1-\pi_*)$, the trivial rate of the prior-only classifier.
\end{proof}

\subsection{Proof of Corollary~\ref{cor:phase-separation}}
\begin{proof}[Proof of Corollary~\ref{cor:phase-separation}]
Fix $(\pi,\beta_0,u)\in[\varepsilon,1-\varepsilon]\times \cB\times \mathbb{S}^{p-1}$. Choose a deterministic sequence $a_n$ such that
\[
a_n\to\infty, \qquad a_n=o\left(\frac{n}{\log n}\right)^{1/4},
\]
and define
\[
\rho_n := \left(\frac{\log n}{n}\right)^{1/4} a_n.
\]
Then $\rho_n\to0$. Choose $\Delta_n$ so that $u(\Delta_n)=u$ and $\rho(\Delta_n)=\rho_n$, and set $\theta_n=(\pi,\beta_0,\Delta_n)$.

The quartic detectability law gives
\[
D_m^*(\theta_n) = J_m(\pi,\beta_0,u)\rho_n^4 + o(\rho_n^4).
\]
Hence
\[
\frac{nD_m^*(\theta_n)}{\log n} = J_m(\pi,\beta_0,u)\frac{n\rho_n^4}{\log n} + o\left(\frac{n\rho_n^4}{\log n}\right) = J_m(\pi,\beta_0,u)a_n^4+o(a_n^4).
\]
Assumption~\ref{ass:A3} gives $J_m(\pi,\beta_0,u)>0$, and since $a_n\to\infty$,
\[
\frac{nD_m^*(\theta_n)}{\log n}\to\infty.
\]

The quadratic recoverability law gives
\[
R_m^*(\theta_n) = A_m(\pi,\beta_0,u)\rho_n^2 + o(\rho_n^2) = A_m(\pi,\beta_0,u) \left(\frac{\log n}{n}\right)^{1/2} a_n^2 + o\left( \left(\frac{\log n}{n}\right)^{1/2} a_n^2 \right).
\]
Because $a_n^2=o\left(n/\log n\right)^{1/2}$, the leading term tends to zero, and thus $R_m^*(\theta_n)\to0$.

These two displays give the claimed phase separation. Under the $K=1$ fit, the best single-component approximation loses population log-likelihood by $D_m^*(\theta_n)$ per observation, so the total likelihood gap is $nD_m^*(\theta_n)\{1+o(1)\}$. The BIC comparison between $K=2$ and $K=1$ subtracts a penalty of order $\log n$. Since $nD_m^*(\theta_n)/\log n\to\infty$, the likelihood gain dominates the BIC penalty, and any BIC-consistent selector chooses $K=2$ with probability tending to one.

Since $R_m^*(\theta_n)\to0$, the entropy expansion (combining Theorem~\ref{thm:quadratic-recoverability} and Proposition~\ref{prop:entropy-representation}) yields $E_m(\theta_n)\to h(\pi)$, and Proposition~\ref{thm:oracle-impossibility} implies that even the oracle Bayes classifier cannot achieve vanishing label error.
\end{proof}


\section{Proofs for Section~\ref{sec:inference}}
\label{sec:SM-inference}

\subsection{Proof of Theorem~\ref{thm:order-selection}}
\begin{proof}
Throughout the proof, we let $K \in \{1, 2\}$ and $\widehat E_{m,1} \equiv 0$.

\emph{Case 1: $K_0 = 2$ (the underfitting direction $K = 1$).} Under (A4), $D_{m,2}^*(\theta_0) \ge \underline D > 0$, so the standard mixture-model identification argument gives
\[
\ell_n(\hat\theta^{\MLE}_{n,2}) - \ell_n(\hat\theta^{\MLE}_{n,1}) = n D_{m,2}^*(\theta_0)\{1 + o_p(1)\} = \Omega_p(n).
\]
Hence
\[
{\rm BIC}(1) - {\rm BIC}(2) = 2 n D_{m,2}^*(\theta_0) - (\nu_2 - \nu_1) \log n + o_p(n) \to +\infty.
\]
For RA-BIC, $\widehat E_{m,1} = 0$ and $\widehat E_{m,2} \le n h(\hat\pi)$, so
\[
\lambda_n \{\widehat E_{m,1} - \widehat E_{m,2}\} \ge -\lambda_n n h(\hat\pi) = -O_p(\sqrt{n \log n}) = o(n).
\]
Therefore $\RABIC(1) - \RABIC(2) = \Omega_p(n) + o(n) \to +\infty$, and $\Pr(\widehat K_{\mathrm{RA}} = 2) \to 1$.

\emph{Case 2: $K_0 = 1$ (the overfitting direction $K = 2$).} Under (A6), $\ell_n(\hat\theta^{\MLE}_{n,2}) - \ell_n(\hat\theta^{\MLE}_{n,1}) = O_p(1)$. The BIC complexity penalty difference is $(\nu_2 - \nu_1) \log n \to +\infty$. Since $\widehat E_{m,1} = 0$ and $\widehat E_{m,2} \ge 0$,
\[
\lambda_n \{\widehat E_{m,2} - \widehat E_{m,1}\} = \lambda_n \widehat E_{m,2} \ge 0.
\]
Therefore
\[
\RABIC(2) - \RABIC(1) = -O_p(1) + (\nu_2 - \nu_1) \log n + \lambda_n \widehat E_{m,2} \to +\infty,
\]
and $\Pr(\widehat K_{\mathrm{RA}} = 1) \to 1$.
\end{proof}

\subsection{Proof of Theorem~\ref{thm:gap-rejection}}
\begin{proof}
By Theorem~\ref{thm:true-quartic}, $D_m^*(\theta_n) = J_m \rho_n^4 + o(\rho_n^4) = J_m a_n^4 \log n / n \cdot (1 + o(1))$, so
\[
\frac{n D_m^*(\theta_n)}{\log n} = J_m a_n^4 (1 + o(1)) \to \infty,
\]
and the BIC log-likelihood gap satisfies
\[
2 \{\ell_n(\hat\theta_{n,2}^{\MLE}) - \ell_n(\hat\theta_{n,1}^{\MLE})\} = 2 J_m a_n^4 \log n \{1 + o_p(1)\}.
\]
Since $a_n \to \infty$, $2 J_m a_n^4 \log n \gg (\nu_2 - \nu_1) \log n$, so $\Pr({\rm BIC}(1) > {\rm BIC}(2)) \to 1$.

For RA-BIC, by Theorem~\ref{thm:quadratic-recoverability}, $R_m^*(\theta_n) = A_m \rho_n^2 + o(\rho_n^2) = A_m a_n^2 \sqrt{\log n / n} (1 + o(1)) \to 0$, so $\widehat E_{m,2}/n = h(\hat\pi) - R_m^*(\theta_n) + o_p(1) \to h(\pi)$. The entropy penalty difference satisfies
\[
\lambda_n \widehat E_{m,2} = h(\pi) \sqrt{n \log n} \{1 + o_p(1)\}.
\]
Comparing with the BIC likelihood gain $2 J_m a_n^4 \log n$:
\[
\frac{2 J_m a_n^4 \log n}{h(\pi) \sqrt{n \log n}}
= \frac{2 J_m}{h(\pi)} \cdot \frac{a_n^4}{\sqrt{n / \log n}} \to 0,
\]
where the last limit uses $a_n^4 = o(\sqrt{n / \log n})$, equivalently $a_n = o((n/\log n)^{1/8})$. Hence $\lambda_n \widehat E_{m,2}$ dominates the BIC likelihood gain, and
\begin{align*}
&\RABIC(1) - \RABIC(2) \\
& \quad = -2 J_m a_n^4 \log n - (\nu_2 - \nu_1) \log n + h(\pi) \sqrt{n \log n}\{1 + o_p(1)\} \to -\infty,
\end{align*}
so $\Pr(\RABIC(1) < \RABIC(2)) \to 1$ holds.
\end{proof}


\section{Details on the entropy-regularized estimator}
\label{sec:SM-ER-calibration}

This section supplies the formal basis for the design choice $\alpha_n = \sqrt{\log n / n}$ of Section~\ref{sec:estimation}. 
We present two propositions and an interpretive paragraph.
The first records that the penalized objective $Q_\alpha = \ell_n + \alpha\widehat E_m$ is, to first order in $\alpha$, the stationarity condition of tempered (deterministic-annealing) EM, so the two optimizers coincide locally and either may be used. The second establishes that, under an interior signal and (A1)--(A4), the design choice $\alpha_n = \sqrt{\log n/n}$ keeps $\widehat\theta_{\mathrm{ER}}$ consistent in two steps: first a localization-and-consistency argument that requires no score expansion (the bounded penalty $\alpha_n \widehat E_m/n \to 0$ uniformly, so the penalized objective converges uniformly to the same population objective as the likelihood, whose well-separated maximizer is $\theta_0$), and then, with consistency in hand, a standard $Z$-estimator expansion about $\widehat\theta_{\rm MLE}$ that establishes the upper bound $\widehat\theta_{\mathrm{ER}} - \widehat\theta_{\rm MLE} = O_p(\sqrt{\log n/n})$, the entropy penalty perturbing the score by $O_p(n\alpha_n)$. 
Here $\alpha_n$ is matched to the RA-BIC rate $\lambda_n$ and to the recoverability scale
$R_m^*(\theta)=h(\pi)-E_m(\theta)$ as a design choice.
It is not derived from a selected-maximum or pointwise-dominance argument, and we make no claim of sharp dominance.

Throughout, we work in the weighted-center parametrization of Section~\ref{sec:model} at a fixed interior parameter $\theta_0 = (\pi_0, \beta_0, \Delta_0)$ with separation magnitude $\rho_0 := \rho(\Delta_0) > 0$ and recoverability $R_m^*(\theta_0) > 0$ (interior signal), and we restrict attention to trial counts $m \geq 2$, where the over-dispersion does not provide a direct route to label recovery. 
Regularity conditions (A1)--(A3) hold, the interior-signal condition (A4) is in force as stated above, and where the recoverability slope is needed we invoke (R1)--(R3) of Assumption~\ref{ass:SM-technical}.

\subsection{First-order equivalence with tempered EM}
\label{sec:SM-ER-tempered}

The penalized objective $Q_\alpha(\theta) = \ell_n(\theta) + \alpha\widehat E_m(\theta)$ of \eqref{eq:regularized-estimator} is, to first order in $\alpha$, the stationarity condition of tempered EM. Write the tempered (deterministic-annealing) free energy at inverse temperature $\omega$ as
\[
\ell_\omega(\theta) := \frac{1}{\omega}\sum_{i=1}^n \log \sum_{k=1}^K \bigl\{\pi_k f(Y_i | X_i, \beta_k)\bigr\}^{\omega},
\]
so that $\ell_1 = \ell_n$ is the ordinary log-likelihood and decreasing $\omega$ below $1$ softens the posterior responsibilities \citep{ueda1998deterministic}.

\begin{proposition}[First-order equivalence]
\label{prop:SM-ER-tempered}
The temperature derivative of $\ell_\omega$ at $\omega = 1$ equals the negative empirical posterior entropy,
\[
\frac{d}{d\omega}\ell_\omega(\theta)\Big|_{\omega = 1} = -\widehat E_m(\theta),
\]
so that, writing $\omega = 1/(1+\alpha)$, the tempered objective $\ell_\omega$ has the same stationary point to first order in $\alpha$ as $\ell_n + \alpha'\widehat E_m$ with $\alpha' = \alpha + O(\alpha^2)$. Hence $\widehat\theta_{\mathrm{ER}}$ and the tempered-EM fixed point at $\omega = 1/(1+\alpha)$ agree up to an $O(\alpha^2)$ discrepancy. At the design scale $\alpha_n = \sqrt{\log n/n}$ this is of order $\log n/n$ and is dominated by the $O_p(\sqrt{\log n/n})$ deviation of either estimator from the MLE established below, so either may be used as the optimizer to that order.
\end{proposition}

\begin{proof}
For a single observation write $s_i(\omega) := \sum_{k=1}^K w_{ik}^\omega$ with $w_{ik} := \pi_k f(Y_i | X_i, \beta_k) > 0$, and let $r_{ik}(\omega) := w_{ik}^\omega / s_i(\omega)$ denote the tempered responsibilities, which satisfy $\sum_{k=1}^K r_{ik}(\omega) = 1$. Differentiating $\omega^{-1}\log s_i(\omega)$ in $\omega$,
\begin{align*}
\frac{d}{d\omega}\frac{\log s_i(\omega)}{\omega}
&= -\frac{\log s_i(\omega)}{\omega^2} + \frac{1}{\omega}\frac{\sum_{k=1}^K w_{ik}^\omega \log w_{ik}}{s_i(\omega)}\\
&= \frac{1}{\omega^2}\sum_{k=1}^K r_{ik}(\omega)\log \frac{w_{ik}^\omega}{s_i(\omega)}
= \frac{1}{\omega}\sum_{k=1}^K r_{ik}(\omega)\log r_{ik}(\omega),
\end{align*}
using $\log(w_{ik}^\omega/s_i(\omega)) = \log r_{ik}(\omega)$. At $\omega = 1$ the responsibilities reduce to the ordinary posterior weights $r_{ik}(1) = P_\theta(Z_i = k | X_i, Y_i)$, so the right-hand side is $\sum_{k=1}^K r_{ik}(1)\log r_{ik}(1)$, the negative per-observation posterior entropy. Summing over $i$ gives $\frac{d}{d\omega}\ell_\omega(\theta)|_{\omega=1} = -\widehat E_m(\theta)$. A first-order Taylor expansion of $\ell_\omega$ about $\omega = 1$ in the deviation $\omega - 1 = -\alpha/(1+\alpha) = -\alpha + O(\alpha^2)$ then gives $\ell_\omega(\theta) = \ell_n(\theta) + \alpha\widehat E_m(\theta) + O(\alpha^2)$, whose stationary condition coincides, to first order in $\alpha$, with that of $\ell_n + \alpha\widehat E_m$. The $O(\alpha^2)$ remainder of $\ell_\omega$ is a log-weight-variance term $\frac12\sum_i r_{i1}r_{i2}(\log w_{i1}-\log w_{i2})^2$ that is not proportional to $\widehat E_m$, so the matching is intrinsically first-order, and the two stationary points agree up to $O(\alpha^2)$.
\end{proof}

Proposition~\ref{prop:SM-ER-tempered} formalizes the remark in Section~\ref{sec:estimation} that the penalized objective and tempered EM are interchangeable optimizers at the scale $\alpha_n \to 0$ we use. The $O(\alpha^2) = O(\log n / n)$ discrepancy between the two stationary points is of smaller order than the $O_p(\sqrt{\log n/n})$ deviation of either from the MLE established next, so it is immaterial to the calibration. Tempered EM is the more stable optimizer in practice because the softened E-step avoids the degenerate responsibilities that an entropy gradient evaluated at near-collapsed weights can produce.

\subsection{Consistency and deviation rate at the design scale}
\label{sec:SM-ER-rate}

\begin{proposition}[Consistency and rate]
\label{prop:SM-ER-rate}
Let $m \geq 2$ and suppose (A1)--(A4) hold at an interior signal $\theta_0$, so that in particular $R_m^*(\theta_0) \geq \underline R > 0$ (recoverability) and $D_m^*(\theta_0) \geq \underline D > 0$ (detectability, separating $\theta_0$ from the one-component/merge boundary). Take the design scale $\alpha_n = \sqrt{\log n / n}$. Then $\widehat\theta_{\mathrm{ER}} = \argmax_\theta Q_{\alpha_n}(\theta)$ is consistent, $\widehat\theta_{\mathrm{ER}} \to \theta_0$ in probability, and the deviation from the MLE satisfies the upper bound
\[
\widehat\theta_{\mathrm{ER}} - \widehat\theta_{\rm MLE} = O_p\bigl(\sqrt{\log n / n}\bigr).
\]
If, in addition, the population per-observation entropy-gradient mass at $\theta_0$ is nonzero, $\E\bigl[\nabla e_m(\theta_0)\bigr] \neq 0$, where $e_m(\theta)$ is the per-observation summand of the empirical posterior entropy $\widehat E_m(\theta)$ of \eqref{eq:entropy} (so $\widehat E_m = \sum_{i=1}^n e_m$ and $\E[\nabla e_m(\theta_0)]$ is its population gradient mass at $\theta_0$), then the order is exact, $\|\widehat\theta_{\mathrm{ER}} - \widehat\theta_{\rm MLE}\| \asymp_p \sqrt{\log n / n}$, so $\widehat\theta_{\mathrm{ER}}$ is not first-order asymptotically equivalent to $\widehat\theta_{\rm MLE}$.
\end{proposition}

\begin{proof}
Write the penalized score as $\Psi_n(\theta) := \nabla \ell_n(\theta) + \alpha_n \nabla \widehat E_m(\theta)$, so that $\widehat\theta_{\mathrm{ER}}$ solves $\Psi_n(\theta) = 0$ and $\widehat\theta_{\rm MLE}$ solves $\nabla \ell_n(\theta) = 0$. We establish localization and consistency \emph{first}, by a direct argmax argument that uses no score expansion, and only then expand the score. This ordering is essential, since the $Z$-estimator expansion below is licensed only once $\widehat\theta_{\mathrm{ER}}$ is already known to lie in a shrinking neighborhood of $\widehat\theta_{\rm MLE}$.

\emph{Step 1 (localization and consistency, no expansion).} Divide the penalized objective by $n$:
\[
\frac{Q_{\alpha_n}(\theta)}{n} = \frac{\ell_n(\theta)}{n} + \alpha_n\frac{\widehat E_m(\theta)}{n}.
\]
The posterior entropy is bounded observation-by-observation. Each per-observation posterior entropy $e_m$ is the Shannon entropy of a two-point posterior over the component label and so lies deterministically in $[0, \log 2]$, whence $0 \le \widehat E_m(\theta)/n \le \log 2$ uniformly in $\theta$ (a $\theta$-uniform, moment-free bound consistent with $\widehat E_m = O_p(n)$). Since $\alpha_n = \sqrt{\log n/n} \to 0$, the penalty term satisfies
\[
\sup_{\theta \in \Theta}\;\alpha_n\frac{\widehat E_m(\theta)}{n} \;\le\; \alpha_n \log 2 \;\to\; 0,
\]
so it vanishes \emph{uniformly} in $\theta$, with no appeal to any gradient bound, rate, or score expansion. Under (A2) the unpenalized criterion converges uniformly over the compact set $\Theta$ of (A1), $\ell_n(\theta)/n \to M(\theta)$ in probability uniformly (a uniform law of large numbers / Glivenko--Cantelli argument under the $8$th-moment envelope of (A2)), where $M(\theta)$ is the population objective. By the triangle inequality $\sup_{\theta \in \Theta} |Q_{\alpha_n}(\theta)/n - \ell_n(\theta)/n| \le \alpha_n \log 2 \to 0$, so $Q_{\alpha_n}/n$ has the same uniform limit $M(\theta)$. The two-component logistic mixture is identifiable up to label permutation, and the detectability margin $D_m^*(\theta_0) \geq \underline D > 0$ of (A4) separates $\theta_0$ from the one-component/merge boundary, so the population Kullback--Leibler projection is uniquely minimized at $\theta_0$ (up to the label-swap equivalence class, which the weighted-center parametrization of Section~\ref{sec:model} fixes), i.e.\ $M$ has a well-separated maximizer at $\theta_0$. Uniform convergence of the criterion together with a well-separated maximum is exactly the Wald/van der Vaart argmax-consistency condition \citep[Thm.~5.7]{vandervaart1998asymptotic}, so $\widehat\theta_{\mathrm{ER}} \to_p \theta_0$. The same conditions give $\widehat\theta_{\rm MLE} \to_p \theta_0$ (the MLE consistency already invoked under (A4)), so both estimators converge to $\theta_0$ and in particular $\|\widehat\theta_{\mathrm{ER}} - \widehat\theta_{\rm MLE}\| \to_p 0$, the localization needed in Step 2. This derivation references neither the deviation rate nor the expansion of Step 2.

The one subtlety is that $\widehat E_m(\theta)$ is itself maximized \emph{toward} the merge boundary $\Delta = 0$, where the posterior distribution of the label is most diffuse. Thus, one must rule out the possibility that the vanishing entropy penalty pulls the maximizer to that higher-entropy boundary. 
This possibility is ruled out by the interior-signal condition.
The detectability margin of (A4) gives an $O(1)$ likelihood separation of $\theta_0$ from the boundary, $M(\theta_0) - \sup_{\theta : \Delta = 0} M(\theta) = D_m^*(\theta_0) \geq \underline D > 0$, whereas the entropy tilt the penalty can add anywhere is at most $\alpha_n \log 2 = o(1)$. For $n$ large enough that $\alpha_n \log 2 < \underline D/2$, no boundary configuration can overtake a neighborhood of $\theta_0$ in $Q_{\alpha_n}/n$, so the argmax stays bounded away from $\Delta = 0$. 
This preserves the well-separated-maximum condition even though the entropy term favors the merge boundary.
The $O(1)$ detectability margin dominates the $o(1)$ penalty tilt.

\emph{Step 2 (deviation rate after localization).} 
With consistency in hand from Step 1, $\widehat\theta_{\mathrm{ER}}$ lies in a shrinking neighborhood of $\widehat\theta_{\rm MLE}$ with probability tending to one, so the $Z$-estimator (Taylor) expansion of $\Psi_n$ about $\widehat\theta_{\rm MLE}$ is valid, with a remainder that is $o_p\bigl(n\|\widehat\theta_{\mathrm{ER}} - \widehat\theta_{\rm MLE}\|\bigr)$ because $\|\widehat\theta_{\mathrm{ER}} - \widehat\theta_{\rm MLE}\| \to_p 0$ is already established rather than assumed. Under (A1)--(A4) and interior signal, $\widehat E_m(\theta)$ has a gradient that is bounded in probability on the aggregate scale uniformly over a neighborhood of $\theta_0$, namely $\nabla \widehat E_m(\theta) = O_p(n)$, since $\widehat E_m$ is a sum of $n$ uniformly bounded per-observation posterior entropies with derivatives controlled by the moment conditions (A2). The per-observation entropy gradient has the form $-\bigl[\operatorname{logit}(r_i)r_i(1-r_i)\bigr]\nabla(\text{log-odds})_i$, whose scalar prefactor is bounded uniformly over the whole simplex (the factor \(r_i(1-r_i)\) controls the logarithmic singularity, with $|\operatorname{logit}(r)r(1-r)| \le 1/4$), so $\nabla \widehat E_m = O_p(n)$ holds even at near-collapsed responsibilities and even at $m=1$, requiring only a finite first moment of $m\|X\|$ (supplied by (A2)). The role of $m \geq 2$ and the interior signal is not to bound this gradient but to secure the nonsingular limiting information $I(\theta_0) \succ 0$ used in the Hessian step below. Hence the entropy penalty perturbs the score by
\[
\alpha_n \nabla \widehat E_m(\theta) = O_p(n\alpha_n) = O_p\bigl(\sqrt{n \log n}\bigr).
\]
The expansion gives
\[
0 = \Psi_n(\widehat\theta_{\mathrm{ER}}) = \nabla \ell_n(\widehat\theta_{\rm MLE}) + \nabla^2 \ell_n(\widehat\theta_{\rm MLE})(\widehat\theta_{\mathrm{ER}} - \widehat\theta_{\rm MLE}) + \alpha_n \nabla \widehat E_m(\widehat\theta_{\rm MLE}) + o_p\bigl(n\|\widehat\theta_{\mathrm{ER}} - \widehat\theta_{\rm MLE}\|\bigr).
\]
Since $\nabla \ell_n(\widehat\theta_{\rm MLE}) = 0$ and, by the interior-signal nonsingularity of the information, $-\nabla^2 \ell_n(\widehat\theta_{\rm MLE}) = n I(\theta_0)\{1 + o_p(1)\}$ with $I(\theta_0) \succ 0$, solving for the deviation gives
\begin{align*}
\widehat\theta_{\mathrm{ER}} - \widehat\theta_{\rm MLE} 
&= \bigl(n I(\theta_0)\bigr)^{-1} \alpha_n \nabla \widehat E_m(\widehat\theta_{\rm MLE})\{1 + o_p(1)\} \\
&= O_p\bigl(n^{-1} \sqrt{n\log n}\bigr) = O_p\bigl(\sqrt{\log n/n}\bigr),
\end{align*}
which is the asserted upper bound. This bound is tight precisely when the leading term does not vanish in probability. By the consistency of the MLE together with a local uniform law of large numbers and continuity of $\theta \mapsto \E[\nabla e_m(\theta)]$ near $\theta_0$ (the same moment control from (A2) used for the upper bound), $n^{-1}\nabla \widehat E_m(\widehat\theta_{\rm MLE}) \to_p \E[\nabla e_m(\theta_0)]$, evaluated at the random argument $\widehat\theta_{\rm MLE}$. Hence if $\E[\nabla e_m(\theta_0)] \neq 0$ then $\nabla \widehat E_m(\widehat\theta_{\rm MLE}) \asymp_p n$, and since $I(\theta_0) \succ 0$ is invertible the nonvanishing survives the multiply by $(n I(\theta_0))^{-1}$, so $(n I(\theta_0))^{-1} \alpha_n \nabla \widehat E_m(\widehat\theta_{\rm MLE}) \asymp_p \alpha_n$ and the deviation is of exact order $\sqrt{\log n/n}$, whence $\sqrt n(\widehat\theta_{\mathrm{ER}} - \widehat\theta_{\rm MLE}) \asymp_p \sqrt{\log n}$ does not vanish and $\widehat\theta_{\mathrm{ER}}$ is not first-order asymptotically equivalent to $\widehat\theta_{\rm MLE}$. Without this nonvanishing condition only the upper bound is claimed.
\end{proof}

The rate $\sqrt{\log n/n}$ is slower than the parametric $n^{-1/2}$ scale of the MLE itself, since $\sqrt n(\widehat\theta_{\mathrm{ER}} - \widehat\theta_{\rm MLE}) = O_p(\sqrt{\log n})$ does not vanish, so $\widehat\theta_{\mathrm{ER}}$ is not first-order asymptotically equivalent to the MLE. 
This choice is intentional.
We set $\alpha_n$ to the RA-BIC rate $\lambda_n$ so that the estimation-side correction matches the selection-side penalty, and the persistent $\sqrt{\log n}$ inflation over the pointwise $n^{-1/2}$ scale is exactly the scale at which the entropy penalty calibrates the posterior responsibilities against the recoverability scale $R_m^*(\theta) = h(\pi) - E_m(\theta)$. The correction is therefore largest at small and moderate $n$ and vanishes as $\alpha_n \to 0$, consistent with the numerical evidence of Section~\ref{sec:experiments}.

\subsection{Entropy regularization as local shrinkage toward the merge boundary}
\label{sec:SM-ER-positioning}

The entropy-regularized estimator can be read as a recoverability-calibrated instance of local-asymptotic submodel shrinkage toward the singular boundary $\Delta = 0$ at which the two components merge. In the $O_p(\sqrt{\log n/n})$ neighborhood of the MLE, the penalty $\alpha_n \nabla \widehat E_m$ pulls the fit toward configurations with higher posterior entropy, i.e., away from the over-separated boundary direction in which the MLE over-concentrates. This is the mixture analogue of shrinkage toward a lower-dimensional submodel, where a Stein-type effect ($p \geq 3$) makes shrinkage toward the constraint surface advantageous in mean squared error for the relevant component parameters \citep{hansen2016shrinkage}. The boundary $\Delta = 0$ is precisely the singular set at which the standard merge-testing asymptotics degenerate \citep{chen1995optimal}, so calibrating the shrinkage to the recoverability scale rather than to an arbitrary constant is what keeps the correction interpretable as a posterior-calibration device. 
Any dominance of $\widehat\theta_{\mathrm{ER}}$ over the MLE is \emph{local}, confined to $m \geq 2$ and small-to-moderate $n$, where the over-concentration is sharpest, and we do not claim a new sharp dominance theorem derivable from the quartic-detectability/quadratic-recoverability gap alone. The empirical evidence of Section~\ref{sec:experiments} is the primary support for the practical effect.


\section{Supplementary numerical results}
\label{sec:SM-numerical}

Table~\ref{tab:calibration-n5000} replicates the calibration experiment of Table~\ref{tab:calibration} at $n = 5000$ (where $\alpha_n = \sqrt{\log n / n} \approx 0.041$). The qualitative pattern is preserved. Deep-gap MLE over-collapse persists, oracle posteriors show essentially zero collapse, and ER reduces the MLE collapse rate by a substantial multiplicative factor. The absolute effect sizes are nonetheless smaller than at $n = 200$, consistent with the asymptotic vanishing of the entropy penalty.

\begin{table}[htb!]
\centering
\caption{Posterior calibration at $n=5000$, 50 replications, $n_{\mathrm{init}}=20$. The qualitative pattern of Table~\ref{tab:calibration} ($n = 200$) is preserved with smaller absolute effect size. In the deep gap regime $(m,\rho) = (50, 0.10)$, the MLE collapse rate is still $19.1\%$ versus $0.01\%$ for the oracle, and ER reduces it to $11.2\%$.}
\label{tab:calibration-n5000}

\smallskip
\begin{tabular}{ccc|ccc|ccc|ccc}
\toprule
& & & \multicolumn{3}{c|}{Collapse rate (\%)} & \multicolumn{3}{c|}{Brier} & \multicolumn{3}{c}{Log loss} \\
$m$ & $\rho$ & regime & Or. & MLE & ER & Or. & MLE & ER & Or. & MLE & ER \\
\midrule
5  & 0.30 & gap entry   & 0.0  & 0.3  & 0.0  & 0.244 & 0.281 & 0.261 & 0.694 & 0.770 & 0.719 \\
5  & 0.50 & gap         & 0.2  & 0.7  & 0.2  & 0.235 & 0.255 & 0.248 & 0.667 & 0.708 & 0.689 \\
5  & 1.00 & recoverable & 5.3  & 5.5  & 5.0  & 0.204 & 0.206 & 0.206 & 0.593 & 0.593 & 0.593 \\
20 & 0.15 & deep gap    & 0.0  & 7.8  & \textbf{3.7}  & 0.244 & 0.326 & \textbf{0.306} & 0.694 & 0.976 & \textbf{0.866} \\
20 & 0.25 & gap         & 0.3  & 1.2  & 0.5  & 0.235 & 0.258 & 0.254 & 0.661 & 0.719 & 0.706 \\
20 & 0.50 & recoverable & 5.7  & 5.8  & 5.4  & 0.204 & 0.205 & 0.206 & 0.586 & 0.590 & 0.591 \\
50 & 0.10 & deep gap    & 0.0  & \textbf{19.1} & \textbf{11.2} & 0.244 & 0.350 & \textbf{0.331} & 0.694 & 1.135 & \textbf{0.997} \\
50 & 0.18 & gap         & 0.6  & 1.8  & 1.2  & 0.231 & 0.258 & 0.256 & 0.654 & 0.721 & 0.713 \\
50 & 0.50 & recoverable & 18.2 & 18.3 & 18.0 & 0.168 & 0.169 & 0.169 & 0.494 & 0.494 & 0.494 \\
\bottomrule
\end{tabular}
\end{table}

\end{document}